\documentclass{article}
\usepackage{iclr2027_conference,times}

\usepackage{amsmath,amssymb,amsfonts,amsthm,bm}
\usepackage{mathtools}
\usepackage{graphicx}
\usepackage{booktabs}
\usepackage{enumitem}
\usepackage{xcolor}
\IfFileExists{placeins.sty}{\usepackage{placeins}}{}
\usepackage{microtype}
\usepackage{algorithm}
\usepackage{algorithmic}
\usepackage{hyperref}
\usepackage{url}
\usepackage{comment}
\hypersetup{hidelinks}

\newtheorem{theorem}{Theorem}
\newtheorem{lemma}{Lemma}
\newtheorem{proposition}{Proposition}
\newtheorem{corollary}{Corollary}

\newcommand{\R}{\mathbb{R}}
\newcommand{\M}{\mathcal{M}}
\newcommand{\D}{\mathcal{D}}
\newcommand{\E}{\mathbb{E}}
\newcommand{\dist}{\operatorname{dist}}
\newcommand{\Id}{I}
\newcommand{\dd}{\mathrm{d}}
\newcommand{\Exp}{\operatorname{Exp}}
\newcommand{\Log}{\operatorname{Log}}
\newcommand{\SO}{\mathrm{SO}}

\newcommand{\skewop}{\operatorname{skew}}
\newcommand{\symop}{\operatorname{sym}}
\newcommand{\argmin}{\operatorname*{arg\,min}}
\newcommand{\proxyM}{\widehat{\M}_{\mathrm{proxy}}}
\newcommand{\LFM}{\mathcal{L}_{\mathrm{FM}}}

\title{Manifold-Stable Flow Matching}

\author{
  Amirhossein Nazerian \\ 
  Department of Mechanical Engineering \\ 
  Colorado State University \\ 
  Fort Collins, CO 80523, USA \\ 
  \texttt{a.nazerian@colostate.edu} 
  \And
  Ali Pezeshki \\ 
  Department of Electrical and Computer Engineering \\ 
  Colorado State University \\ 
  Fort Collins, CO 80523, USA \\ 
  \texttt{ali.pezeshki@colostate.edu} 
  \AND
  Jianguo Zhao  \\ 
  Department of Mechanical Engineering \\ 
  Colorado State University \\ 
  Fort Collins, CO 80523, USA \\ 
  \texttt{jianguo.zhao@colostate.edu} 
}

\iclrfinalcopy 

\begin{document}

\maketitle

\lhead{}

\begin{abstract}
Flow matching (FM) learns generative dynamics through velocity regression. Geometric FM variants commonly assume a prior supported on the data manifold, requiring geometric knowledge that is often unavailable. Without such knowledge, low regression error alone does not guarantee manifold adherence. Adherence keeps generated samples within valid configurations and is empirically associated with better task performance. We introduce manifold-stable flow matching (MSFM), which can start from an arbitrary ambient prior, not necessarily supported on the manifold. Using tools from nonlinear dynamics, namely contraction theory, MSFM combines learned tangential transport with prescribed normal contraction. The construction uses analytical projectors for known manifolds and local affine proxies estimated by principal component analysis for unknown data geometry. By implementing contraction theory in both cases of known and unknown manifolds, we guarantee manifold invariance and transverse convergence to the manifold within a desired time window (e.g., one second). We derive a family of compatible probability paths and decompose the training loss into a learnable tangential term and a normal residual. An ellipse experiment attains a mean terminal off-manifold error of order $10^{-6}$. In Push-T robotic experiments, MSFM raises success from $74\%$ to $82\%$. In the Robomimic Square task, success increases from $60\%$ to $72\%$, while rotation-manifold deviation decreases from order $10^{-2}$ to $10^{-7}$. The MSFM terminal geometric errors are controlled by the chosen numerical tolerance. These results demonstrate stronger geometric adherence and higher observed task performance, supporting prescribed normal contraction as a complement to learned generative transport.
\end{abstract}

\section{Introduction}
\label{sec:introduction}

Flow matching (FM) learns a time-dependent vector field whose flow transports a simple source distribution to a target data distribution through a regression objective rather than repeated ODE simulation during training \citep{lipman2023flowmatching,albergo2025stochastic}. This framework is especially attractive for generative policies because inference reduces to integrating a learned dynamical system. However, many target distributions in robotics and structured generation are supported on, or concentrated near, a lower-dimensional manifold $\M\subset\R^n$. Examples include rotation matrix parameterizations, rigid-body poses, constrained actions, and data-driven task manifolds.

Adherence to this manifold matters. First, samples off the manifold may not
be valid outputs: a generated rotation that is not in $\mathrm{SO}(3)$ cannot
be executed and must be projected after the fact, and projecting a sample
that has drifted far from the manifold can yield an unintended action.
Second, when the manifold is unknown, as for expert actions in imitation
learning, samples off the data geometry correspond to actions the
demonstrations do not support. In closed-loop control, such errors can
compound, and adherence to the expert action geometry can matter more for
performance than low validation error \citep{pan2026muchado}.

Conventional FM, however, controls only the expected velocity error on sampled
training paths. It matches distributions through training but provides no
mechanism that pulls generated trajectories back to the manifold, so a model with low regression loss can still drift away from it. Existing remedies are
limited: post-hoc projection requires a known manifold and corrects only the
final sample, while geometric FM methods require both a known manifold and a
prior supported on it.

We propose \emph{manifold-stable flow matching} (MSFM), which separates learned transport along the data geometry from prescribed convergence toward it. 
For a known manifold, analytical tangent and normal projectors define this separation, while local principal component analysis (PCA) is used for unknown manifolds. 
Using contraction theory from nonlinear dynamics, we form a normal contractive flow to the data manifold, such that the flow (starting from a prior distribution outside of the manifold) is guaranteed to converge to the manifold within a desired time window (e.g., 1 second).
Figure~\ref{fig:overview} provides a schematic comparison between the conventional FM and our novel MSFM.

Our main contribution is MSFM, a flow-matching framework that provably converges to the data manifold from an arbitrary ambient prior within a finite time window. 
This guarantee holds regardless of regression accuracy since the network learns only tangential transport while normal contraction is prescribed.
We use analytical projectors when the geometry is known and local PCA affine proxies otherwise. 
We support this framework theoretically by deriving a family of compatible tangent--normal probability paths and an orthogonal loss decomposition that separates what the network learns from what is prescribed.
We prove manifold invariance and transverse convergence.
We support it empirically on an ellipse and two robotic manipulation tasks: under matched network capacity, MSFM reduces off-manifold error by multiple orders of magnitude on Push-T and Robomimic Square, and raises observed success.

\begin{figure}[t]
    \centering
    \includegraphics[width=0.85\linewidth]{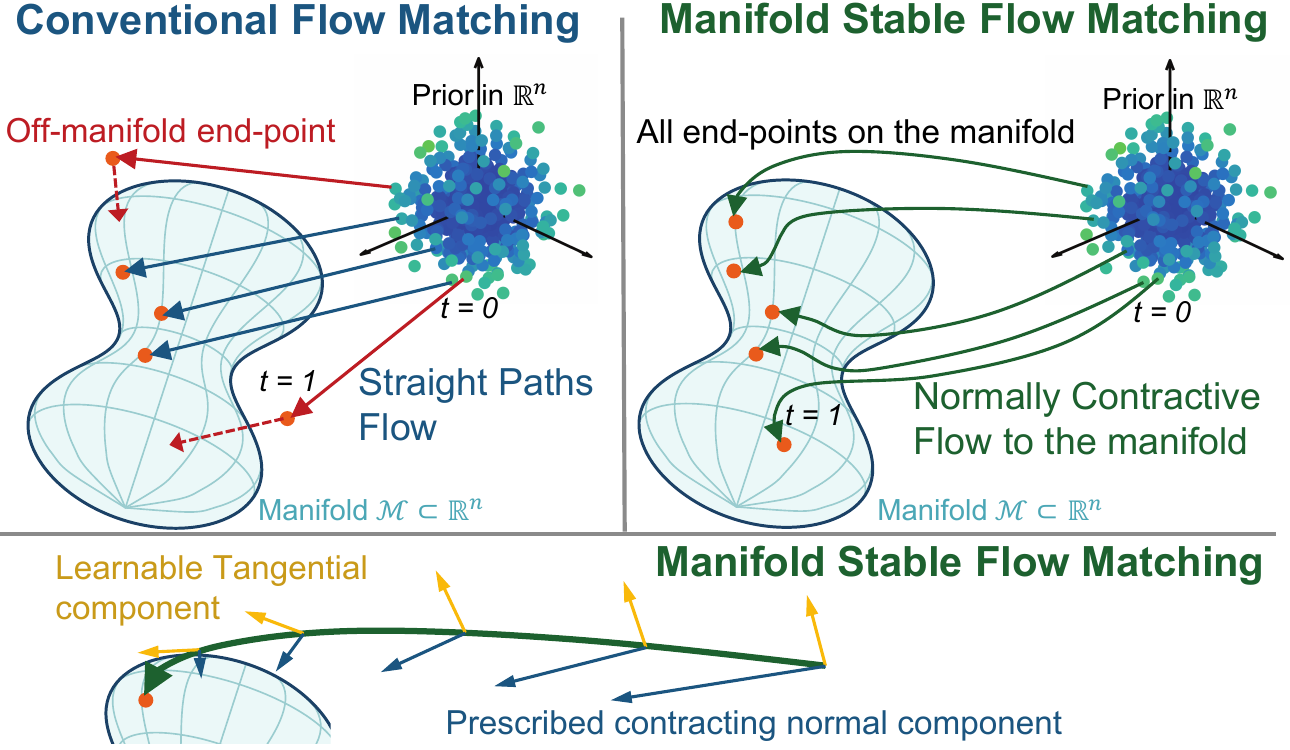}
    \caption{Schematic comparison. Linear conditional FM paths do not necessarily stay near a curved target manifold (top-left panel). 
    MSFM (top-right panel) learns tangential transport and prescribes normal contraction, which guarantees convergence to the manifold.
    The learned MSFM sampling trajectories can curve because MSFM learns tangential motion, while it prescribes normal contraction (bottom panel).
    The normal component of the flow uniquely minimizes kinetic action. 
    }
    \label{fig:overview}
\end{figure}

\section{Related Work}
\label{sec:related}

\paragraph{Intrinsic flow matching.}
Flow matching learns transport through velocity regression \citep{lipman2023flowmatching}.
Riemannian flow matching defines paths and tangent vector fields on a known manifold \citep{chen2023riemannian}, with applications to robot policies and pose estimation \citep{braun2024rfmp,ding2025rfmp,ouyang2026rfmpose}.
Extensions modify the regression objective \citep{zaghen2025towards,zaghen2026riemannian} or target one-step generation \citep{zhong2026riemannian}.
These intrinsic formulations start from manifold-supported priors, often by mapping noise onto the manifold.
Pullback flow matching learns geometry for intrinsic transport but also uses a manifold-adapted source \citep{de2024pullback}.

\paragraph{Learned geometry and ambient-space transport.}
Metric flow matching learns ambient metrics and interpolants that favor proximity to observations \citep{kapusniak2024metric}.
Meta Flow Matching instead studies dynamics on the Wasserstein space of distributions, rather than imposing a geometric constraint on individual samples \citep{atanackovic2025meta}.
Energy matching transports ambient noise using a learned potential \citep{balcerak2026energy}.
Recent works also provided flow matching variations that respect data geometry \citep{bamberger2026carre,jiang2026bures,kumar2026flow,habashy2026geodesic,cai2026improvingclassifierfreeguidanceflow,wu2026flow}.
Flow matching has also been investigated in constrained generation, stability, and domain topology contexts \citep{li2026gauge,ICLR2026_d33a2d61,ICLR2026_58b1b19b,ICLR2026_238e6167,ICLR2026_a57483b3,ICLR2026_e8b782d0}.
For data on a linear subspace, \citet{pi2026learning} use ambient Gaussian noise, retain an analytical normal velocity, and derive statistical convergence bounds.
Such bounds complement geometric stability: finite distributional error does not ensure exact manifold membership or a prescribed trajectory-wise contraction rate, although vanishing Wasserstein error implies concentration near a manifold-supported target.

\paragraph{Stable learned dynamics.}
Contraction theory provides tools for incremental and partial stability \citep{lohmiller1998contraction,wang2005partial,tsukamoto2021tutorial,aminzare2014contraction,FB-CTDS}, with developments in nonlinear control, modular systems, and differential growth analysis \citep{lohmiller2000nonlinear,slotine2001modularity,slotine2003modular,aminzare2013logarithmic,nazerian2024bridging}.
Structured neural dynamics impose contraction \citep{beik2024ncds,jaffe2024elcd}; related stability constructions appear in flow matching \citep{sprague2024stablefm} and LaSalle-based Riemannian policies \citep{ding2025rfmp}.

\paragraph{Generative robotic policies.}
Diffusion Policy and Robomimic provide established settings for generative policies and offline imitation learning \citep{chi2023diffusionpolicy,mandlekar2022robomimic}.
Generative robotic policies have used manifold-valued poses or states \citep{ryu2024diffusionedfs,chatzipantazis2025stride}, spatial equivariance \citep{tie2025etseed,wang2025equivariantpolicy,yang2025equibot,zhu2025spherical}, and projection onto feasible trajectories \citep{bouvier2025ddat}, while empirical work has examined how accurately diffusion policies learn kinematic constraint manifolds \citep{foland2025kinematic}.
Flow-matching policies have also studied point-cloud-conditioned action generation, spatial equivariance, and consistency training for manipulation \citep{chisari2025pointflowmatch,zhang2026e3flow,yan2025maniflow}.

MSFM learns tangential transport while prescribing normal contraction independently of tangential regression accuracy.
Crucially, its prior need not lie on, or first be mapped onto, the target manifold.
This permits analytical projectors for known manifolds or local PCA affine proxies for unknown expert-action geometry, without learning a global intrinsic parameterization.
Under the stated projection and proxy-selection assumptions, contraction theory guarantees convergence from admissible ambient initializations to the known manifold or the modeled proxy geometry, respectively.

\section{Background and Problem Formulation}
\label{sec:background}

We first review the conventional FM.
Let $p_0$ be a source distribution on $\R^n$ and $p_{\mathrm{data}}$ the target distribution. Conditional flow matching samples $x_0\sim p_0$, $x_1\sim p_{\mathrm{data}}$, and $t\sim\mathcal U[0,1]$, constructs $x_t=\gamma(t;x_0,x_1)$ with velocity $u_t=\partial_t\gamma(t;x_0,x_1)$, and minimizes
\begin{equation}
    \LFM(\theta)
    =
    \E\!\left[\|v_\theta(t,x_t)-u_t\|^2\right].
    \label{eq:fm_loss}
\end{equation}
At inference, one draws $x(0)\sim p_0$ and integrates $\dot x=v_\theta(t,x)$.
The usual Euclidean path is $x_t=(1-t)x_0+tx_1$, with $\dot{x}_t=x_1-x_0$. More general scalar interpolations change all ambient directions, including components locally tangent and normal to the data manifold.

\subsection{Embedded-Manifold Geometry}

Let $\M\subset\R^n$ be a smooth embedded manifold. Close enough to $\M$, in a
tubular neighborhood $\mathcal U$, every point $x$ has a unique nearest point
on the manifold, $y=\Pi(x)$, which depends smoothly on $x$. We define an offset $r(x) = x - \Pi(x)=x-y$, which means we can split $x$ into 
this base point and the offset, i.e., $x = y + r(x)$.
At $y$, let $P_y$ project onto the tangent space $T_y\M$ and $Q_y=\Id-P_y$
onto the normal space. The offset $r$ points straight away from the
manifold, so $Q_y r = r$, and its length is the distance to $\M$. We measure
this distance with
\begin{equation}
    V_\M(x)=\tfrac12\dist(x,\M)^2=\tfrac12\|r(x)\|^2,
    \qquad \nabla V_\M(x)=r(x).
    \label{eq:distance_energy_main}
\end{equation}
\paragraph{Why conventional flow matching does not imply stability.}
Along a generated trajectory $\dot x=v_\theta(t,x)$, the distance changes as
$\dot V_\M = r(x)^\top v_\theta(t,x)$: only the normal component of the
velocity matters. Samples are attracted to $\M$ if this rate is negative
everywhere off the manifold, for instance
\begin{equation}
    \dot V_\M = r(x)^\top v_\theta(t,x) \le -\alpha(t)\|r(x)\|^2,
    \qquad \alpha(t)>0,
    \label{eq:transverse_condition}
\end{equation}
which makes the distance shrink at rate at least $\alpha(t)$ at every point
of every trajectory. The regression loss \eqref{eq:fm_loss} only makes the
velocity error small on average over sampled training paths, so it cannot
enforce this pointwise condition: a model with low loss can still drift away
from $\M$ at inference. MSFM enforces \eqref{eq:transverse_condition} by
construction while learning tangential transport.

\section{Manifold-Stable Flow Matching}
\label{sec:method}
MSFM has two ingredients: a \emph{sampling field}, used at inference, that
learns motion along the manifold and prescribes contraction toward it; and a
\emph{training path} that supplies the regression targets. The two must agree
in the normal direction.
We construct both for known manifolds (Section~\ref{sec:known_method}) and for
unknown manifolds via local affine proxies (Section~\ref{sec:proxy_method}).

\subsection{Known-Manifold Construction}
\label{sec:known_method}

For an unconstrained network $w_\theta$ and $y=\Pi(x)$, define
\begin{equation}
    v_\theta(t,x)=P_yw_\theta(t,x)-\alpha(t)r(x).
    \label{eq:known_vector_field}
\end{equation}
The learned component is the tangent $P_yw_\theta(t,x)$; the prescribed normal feedback $-\alpha(t)r(x)$ vanishes on $\M$. The positive, locally integrable rate $\alpha (t)$ is specified in Section~\ref{sec:schedule}.

For source $x_0\in\mathcal U$ and target $x_1\in\M$, construct
\begin{equation}
    x_t=y_t+r_t,\qquad y_t\in\M,\quad r_t\in N_{y_t}\M.
    \label{eq:tubular_path_main}
\end{equation}
Choose a smooth base curve from $y_0=\Pi(x_0)$ to $y_1=x_1$, for example the geodesic $y_t=\Exp_{y_0}(t\Log_{y_0}(x_1))$, using the Riemannian exponential and a suitable logarithm branch. Set $r_0=x_0-y_0$ and $Q_t=Q_{y_t}$, and define
\begin{equation}
    \dot r_t=\dot Q_t r_t-\alpha(t)r_t,\qquad r(0)=r_0.
    \label{eq:normal_path_ode_main}
\end{equation}
Here $\dot Q_t=\dd Q_{y_t}/\dd t$ accounts for the changing normal space, while
\begin{equation*}
    \|r_t\|=q_\perp(t)\|r_0\|,\qquad
    q_\perp(t)=\exp\!\left(-\int_0^t\alpha(\tau)\,\dd\tau\right).
\end{equation*}
For each sampled time $t$, first obtain $r_t$ by evaluating
the solution of \eqref{eq:normal_path_ode_main}, either in
closed form when available or by integrating this prescribed
linear ODE from $0$ to $t$. The training pair is then
\begin{equation*}
    x_t=y_t+r_t,
    \qquad
    u_t=\dot y_t+\dot Q_t r_t-\alpha(t)r_t.
\end{equation*}
For the $\SO(3)$ geometry used in our experiments,
$r_t$ has the explicit expression given in
\eqref{eq:so3_normal_path_appendix} (Appendix~\ref{app:lie_groups}).
Assuming $\Pi(x_t)=y_t$ throughout, sampling $x_0\sim p_0$ defines $p_t(\cdot\mid x_1)$ as the law of $x_t$. 
Lemma~\ref{lem:compatible_paths} in Appendix~\ref{app:compatible_paths} proves that $r_t$ remains normal at $y_t$, with $\|r_t\|=q_\perp(t)\|r_0\|$, and that the target velocity satisfies $Q_tu_t=-\alpha(t)r_t$.
If $y_t\to x_1$ and $\int_0^1\alpha(t)\,\dd t=\infty$, the resulting path converges to $x_1$ as $t\to1^-$.

\subsection{Unknown Manifold: Local Affine Proxies}
\label{sec:proxy_method}

For samples $\D=\{x_i^\ast\}_{i=1}^N$, local PCA supplies orthonormal tangent bases $T_i$ and affine proxies
\begin{equation*}
    P_i=T_iT_i^\top,\quad Q_i=\Id-P_i,\quad
    \widehat\M_i=x_i^\ast+\operatorname{range}(T_i).
\end{equation*}
Define $r_i(x)=Q_i(x-x_i^\ast)$. Select a proxy by nearest anchor or minimum normal residual, $i(x)\in\argmin_i\tfrac12\|r_i(x)\|^2$, and use
\begin{equation}
    v_\theta(t,x)=P_{i(x)}w_\theta(t,x)-\alpha(t)r_{i(x)}(x).
    \label{eq:proxy_vector_field_main}
\end{equation}
This data-derived geometry is related to off-manifold residual metrics \citep{pan2026muchado}. Algorithm~\ref{alg:proxy_preprocessing} gives its construction.

For a fixed proxy containing $x_1$, set $s_0=P_i(x_0-x_1)$ and $r_0=Q_i(x_0-x_1)$. Since $\dot Q_i=0$, a compatible path is
\begin{equation}
    x_t=x_1+(1-t)s_0+q_\perp(t)r_0,\qquad
    u_t=-s_0-\alpha(t)q_\perp(t)r_0.
\label{eq:affine_compatible_path_main}
\end{equation}

\subsection{Algorithms}
\label{sec:algorithms}

Training regresses the tangential velocity and records the normal residual separately. Sampling integrates the learned field from $p_0$, using current-state geometry.

Algorithm~\ref{alg:proxy_preprocessing} is needed only when the manifold is unknown. Algorithms~\ref{alg:proxy_training} and~\ref{alg:proxy_inference} apply to both known and unknown manifolds, through the following shared notation:
\begin{equation}
\label{eq:unifiednotation}
(P(x),r(x))=
\begin{cases}
\bigl(P_{\Pi(x)},\,x-\Pi(x)\bigr),
& \text{known manifold},\\[2pt]
\bigl(P_{i(x)},\,Q_{i(x)}(x-x_{i(x)}^\ast)\bigr),
& \text{unknown manifold},
\end{cases}
\qquad Q(x)=\Id-P(x).
\end{equation}

\begin{algorithm}[ht]
\caption{Local Proxy Preprocessing (Only when the manifold is unknown)}
\label{alg:proxy_preprocessing}
\begin{algorithmic}[1]
\REQUIRE Dataset $\D$, neighbor count $k$, dimension rule
\FOR{each anchor $x_i^\ast\in\D$}
    \STATE Form $Y_i$ from its $k$ nearest-neighbor differences.
    \STATE Let $T_i$ contain the leading $d_i$ left singular vectors of $Y_i$.
    \STATE Set $P_i=T_iT_i^\top$ and $Q_i=\Id-P_i$.
\ENDFOR
\STATE \textbf{return} $\{P_i,Q_i\}_{i=1}^N$.
\end{algorithmic}
\end{algorithm}

\begin{algorithm}[ht]
\caption{MSFM Training}
\label{alg:proxy_training}
\begin{algorithmic}[1]
\REQUIRE $p_0$, data, geometry, rate $\alpha$, cutoff $\delta$, network $w_\theta$
\WHILE{not converged}
    \STATE Sample $(x_0,x_1,t)$ and construct $(x_t,u_t)$ using the chosen path.
    \STATE Evaluate $P$, $Q$, and $r$ at $x_t$ using the sampling geometry.
    \STATE Update $\theta$ to minimize $\|P(w_\theta(t,x_t)-u_t)\|^2$.
    \STATE Record the normal residual $\|Qu_t+\alpha(t)r\|^2$.
\ENDWHILE
\end{algorithmic}
\end{algorithm}

\begin{algorithm}[ht]
\caption{MSFM Sampling}
\label{alg:proxy_inference}
\begin{algorithmic}[1]
\REQUIRE $p_0$, trained $w_\theta$, geometry, rate $\alpha$, cutoff $\delta$
\STATE Draw $x(0)\sim p_0$; hold any observed condition fixed.
\FOR{ODE steps up to $t=1-\delta$}
    \STATE Evaluate the current-state field \eqref{eq:known_vector_field} or \eqref{eq:proxy_vector_field_main}.
    \STATE Advance the numerical solver.
\ENDFOR
\STATE \textbf{return} $x(1-\delta)$.
\end{algorithmic}
\end{algorithm}

\subsection{Contraction Schedule}
\label{sec:schedule}

We use constant plus singular contraction, with $\alpha_0\geq0$ and $\beta>0$:
\begin{align}
    \alpha(t)&=\alpha_0+\frac{\beta}{1-t},\qquad 0\leq t<1,
    \label{eq:singular_schedule_main}
\end{align}
and $q_\perp(t)=e^{-\alpha_0t}(1-t)^\beta$. Since $q_\perp(t)\to0$, normal displacement vanishes at the terminal limit. Numerically, integrate to $1-\delta$ in log time $\tau=-\log(1-t)$, which makes $(1-t)\alpha(t)=\alpha_0e^{-\tau}+\beta$ bounded; details are in Appendix~\ref{app:algorithms}.

\paragraph{Linear-path special case.}
For a fixed affine proxy containing $x_1$, $\alpha_0=0$ and $\beta=1$ reduce \eqref{eq:affine_compatible_path_main} to $x_t=(1-t)x_0+tx_1$ and $u_t=x_1-x_0$. This does not extend to curved-manifold paths.

\paragraph{Physical interpretation.}
The normal potential $V=\tfrac12\|r_t\|^2$ satisfies $\dot V=-2\alpha(t)V$: tangential transport accompanies normal relaxation. For a fixed base curve and $\alpha_0=0$, the normal path also uniquely minimizes a time-weighted normal kinetic action. When $\beta=1$, the weight is constant; the full ambient path can still curve as the base point and normal space move. This is normal-motion optimality, not an optimal-transport claim for the full trajectory. Appendix~\ref{app:normal_kinetic_action} defines the action and proves the result.

Since all contracting dynamical systems are dissipative (i.e., the volume of the ball of initial conditions tends to zero eventually), our MSFM is also normally dissipative: the volume of all perturbations to the initial condition in the normal direction to the manifold will tend to zero.

\section{Theoretical Guarantees}
\label{sec:theory}

\begin{proposition}[Orthogonal MSFM loss decomposition]
\label{prop:loss_decomposition}
At a fixed training pair $(x_t,u_t)$, let $P$ be either geometry's tangent projector in \eqref{eq:unifiednotation}, $Q=\Id-P$, and $r$ its normal residual. For $v_\theta=Pw_\theta-\alpha r$,
\begin{equation}
    \|v_\theta-u_t\|^2
    =\|P(w_\theta-u_t)\|^2+\|Qu_t+\alpha(t)r\|^2.
    \label{eq:loss_decomposition}
\end{equation}
The normal term is parameter-independent and vanishes exactly when $Qu_t=-\alpha(t)r$.
\end{proposition}

Assume $\alpha (t)$ is positive and locally integrable for $t \in [0,1)$, with $\int_0^1\alpha(t)\,\dd t=\infty$ (e.g., as in \eqref{eq:singular_schedule_main}). Solutions must exist up to each $t<1$; for the known-manifold field, assume local uniqueness and that the trajectory remains in the smooth projection neighborhood.
We first show that, for a known manifold, the prescribed normal feedback alone
keeps samples on $\mathcal{M}$ and drives off-manifold samples toward it.

\begin{theorem}[Manifold invariance and convergence]
\label{thm:known_main}
Suppose the solution of \eqref{eq:known_vector_field} remains in $\mathcal U$. If $x(0)\in\M$, then $x(t)\in\M$ for all times of existence. For arbitrary $x(0)\in\mathcal U$,
\begin{equation}
    \dist(x(t),\M)
    \leq
    \exp\!\left(-\int_0^t\alpha(\tau)\,\dd\tau\right)
    \dist(x(0),\M).
    \label{eq:known_bound_main}
\end{equation}
Then $\dist(x(t),\M)\to0$ as $t\to1^-$.
\end{theorem}

The same mechanism applies when the manifold is unknown: using a single local PCA proxy.
\begin{theorem}[Fixed-proxy convergence]
\label{thm:fixed_proxy_main}
Fix $i$ and consider $\dot x=P_iw_\theta(t,x)-\alpha(t)Q_i(x-x_i^\ast)$. Then
\begin{equation}
    \|Q_i(x(t)-x_i^\ast)\|
    \leq
    \exp\!\left(-\int_0^t\alpha(\tau)\,\dd\tau\right)
    \|Q_i(x(0)-x_i^\ast)\|.
    \label{eq:fixed_proxy_bound_main}
\end{equation}
Consequently, $\|Q_i(x(t)-x_i^\ast)\|\to0$ as $t\to1^-$.
\end{theorem}


During sampling, the active proxy may switch as the state moves. Define
$V_{\mathrm{proxy}}(x)=\min_i \tfrac12\|Q_i(x-x_i^*)\|^2$ and
$\widehat{\mathcal{M}}_{\mathrm{proxy}}=\bigcup_i \widehat{\mathcal{M}}_i$.
The next result shows that such switching does not break convergence.

\begin{theorem}[Convergence to the proxy set]
\label{thm:global_proxy_main}
Under the minimum-residual selector $i(x)\in\argmin_i\tfrac12\|r_i(x)\|^2$, any absolutely continuous solution for which \eqref{eq:proxy_vector_field_main} is well defined satisfies
\begin{equation}
    V_{\mathrm{proxy}}(x(t))
    \leq
    \exp\!\left(-2\int_0^t\alpha(\tau)\,\dd\tau\right)
    V_{\mathrm{proxy}}(x(0)).
    \label{eq:global_proxy_bound_main}
\end{equation}
Consequently, $\dist(x(t),\proxyM)$ obeys the corresponding square-root bound and converges to zero.
\end{theorem}

Finally, for $\alpha (t)$ in \eqref{eq:singular_schedule_main}, these bounds take an explicit form that quantifies convergence by $t=1$.
\begin{corollary}[Exponential--polynomial terminal bound]
\label{cor:polynomial_bound}
For \eqref{eq:singular_schedule_main}, let $\mathcal S$ denote either $\M$, or a fixed proxy $\widehat\M_i$, or the proxy union $\proxyM$, as appropriate. 
Then
\begin{equation*}
    \dist(x(t),\mathcal S)
    \leq e^{-\alpha_0t}(1-t)^\beta\dist(x(0),\mathcal S).
\end{equation*}
At the numerical terminal time $1-\delta$, the bound is $\dist(x(1-\delta),\mathcal S)\leq e^{-\alpha_0(1-\delta)}\delta^\beta\dist(x(0),\mathcal S)$.
\end{corollary}

Proofs are given in Appendix~\ref{app:proofs}. Theorems~\ref{thm:known_main}--\ref{thm:global_proxy_main} guarantee geometric stability, which should not be confused with exact matching of the tangential data distribution. 
Under normal path compatibility and exact population regression, the conventional flow-matching marginalization argument applies to the resulting compatible conditional field \citep{lipman2023flowmatching}. 
Without compatibility, the normal residual in \eqref{eq:loss_decomposition} must be treated as a modeling tradeoff rather than hidden inside the training loss.

\section{Experiments}
\label{sec:experiments}

We choose three experiments that together cover both settings of MSFM: an
ellipse, an unknown curved manifold whose analytic form lets us verify
convergence; disturbed Push-T control, an unknown manifold of expert actions
approximated by local PCA proxies; and nominal Robomimic Square
manipulation, a known $\mathrm{SO}(3)$ manifold with analytical projectors.
In each case, we report geometric validity separately from task success.

\begin{figure*}[t]
    \centering
    \includegraphics[width=\textwidth]{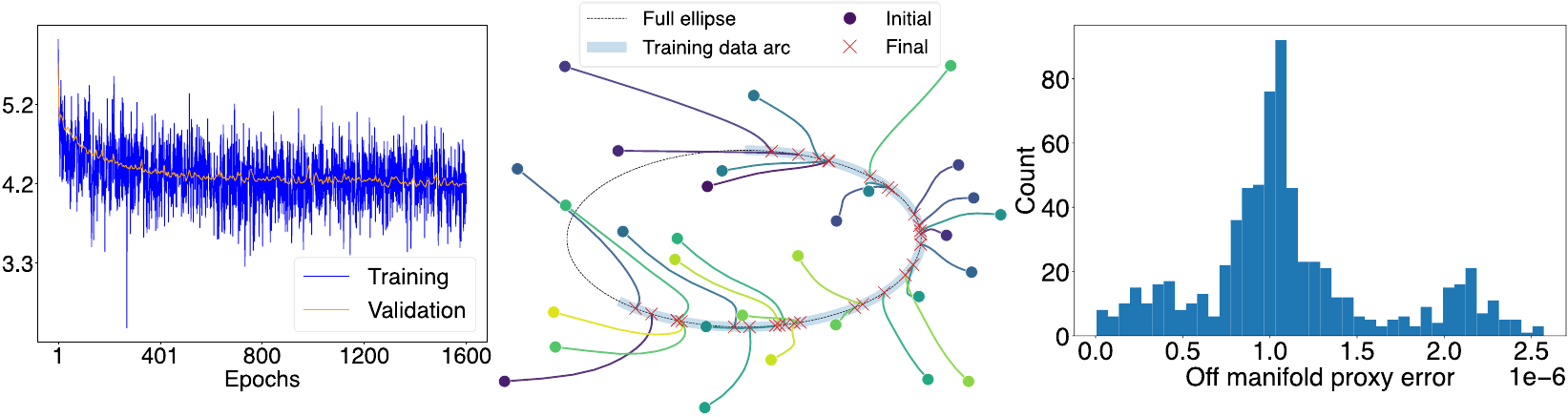}
    \caption{Ellipse: training/validation losses over 1,600 epochs, sampling trajectories, and terminal off-proxy errors for 700 samples. Here, we approximate the manifold (the ellipse) by local affine proxies using PCA.}
    \label{fig:ellipse_summary}
\end{figure*}

\subsection{Ellipse With Local Linear Proxies}
\label{sec:ellipse_main}

The target manifold is the ellipse arc $(2\cos\varphi,\sin\varphi)$, $\varphi\in[-3\pi/4,\pi/2)$, represented by one-dimensional local PCA proxies; its analytic geometry is not supplied to the model. MSFM uses a uniform ambient prior, linear training paths, and $\alpha(t)=1/(1-t)$. Across 700 generated samples, mean terminal off-proxy error is $1.10\times10^{-6}$ and mean analytic ellipse residual is $9.74\times10^{-5}$ (Figure~\ref{fig:ellipse_summary}). These results show adherence to the learned geometry, not distributional equivalence or superiority over FM. Appendix~\ref{app:ellipse} gives data, training, and diagnostic details.

\paragraph{Shared robotic-policy protocol.}
Within the following Push-T and Robomimic Square tasks, FM and MSFM share state observations, action representations, temporal U-Net capacity, Gaussian priors, and AdamW optimization. Both use validation-selected checkpoints, receding-horizon execution, and matched RK4 field-evaluation budgets with method-specific time grids. 
Our U-net has channel widths $(16,32,64)$ for Push-T, and $(32, 60, 124)$ for Robomimic Square. 

\subsection{Push-T}
\label{sec:pusht_main}

Push-T requires a planar pusher to align a T-shaped block with a target pose \citep{chi2023diffusionpolicy,huggingface_gym_pusht}. The target geometry is an unknown manifold of expert action horizons in $\R^{16}$: 8 successive 2D pusher positions. MSFM approximates this geometry by local PCA affine proxies and contracts toward the active proxy while learning tangential motion. It does not impose a manifold on the block's physical pose. Conventional FM uses unconstrained linear-path velocity regression. Both policies condition on the pusher/block state, execute the first target, and replan.

We report task performance only under repeated block-position disturbances. MSFM raises success from $74\%$ to $82\%$ and mean maximum reward from $0.86$ to $0.92$ (Figure~\ref{fig:pusht}). This improvement is consistent with transverse stabilization toward expert-supported action geometry, as shown in Fig.~\ref{fig:pusht}e as an off-manifold proxy error. 
Figure~\ref{fig:pusht}f shows a sample matched episode run in which the MSFM successfully performs the Push-T task, while the conventional FM fails. 
Appendix~\ref{app:pusht} provides additional details

\begin{figure*}[t]
    \centering
    \includegraphics[width=\textwidth]{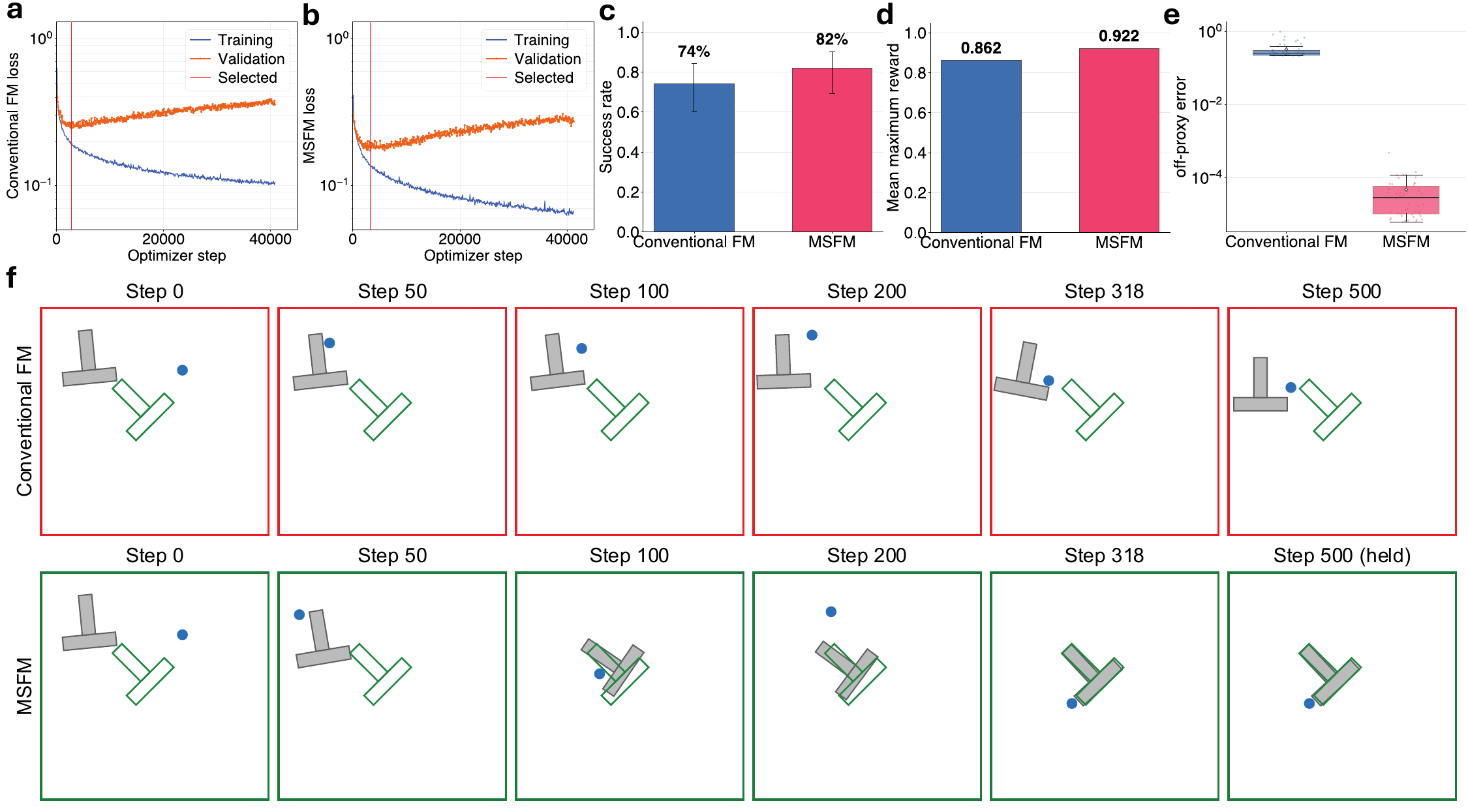}
    \caption{Push-T Task. (a,b) Total FM and parallel-only MSFM training/validation losses; red lines retain the original checkpoint selection. (c,d) Disturbed success and mean maximum reward. (e) Shows that off-manifold-proxy error is 4 orders of magnitude larger for the conventional FM. 
    (f) A matched episode run where MSFM is successful in performing the task, while the conventional FM fails.
    Here, we approximate the manifold (expert actions) by local affine proxies using PCA.}
    \label{fig:pusht}
\end{figure*}

\subsection{Robomimic Square}
\label{sec:square_main}

A Panda robot must place a square nut over its matching peg \citep{mandlekar2022robomimic,zhu2020robosuite}. The known action manifold is $(\R^3\times\SO(3)\times\R)^{64}$: 64-step horizons of translation, rotation, and gripper commands, represented in $\R^{64\times13}$. MSFM constrains the rotation blocks to approach $\SO(3)$ using the tubular path and analytical projectors; translation and gripper components remain unconstrained. Both methods encode absolute pose goals relative to the current end-effector frame, decode with the same final rotation projection, and execute four actions before replanning. FM uses ambient linear paths.

Using the best validation checkpoints, MSFM achieves $72\%$ success versus $60\%$ for FM (Figure~\ref{fig:square}). 
Panel d also shows that, given a budget for the length of an episode, MSFM achieves a higher success rate.
Mean rotation deviation before final projection falls from $10^{-2}$ to $10^{-7}$ in panel e. 
Thus, stronger $\SO(3)$ adherence accompanies higher observed success. Appendix~\ref{app:square} provides additional information.

\begin{figure*}[t]
    \centering
    \includegraphics[width=\textwidth]{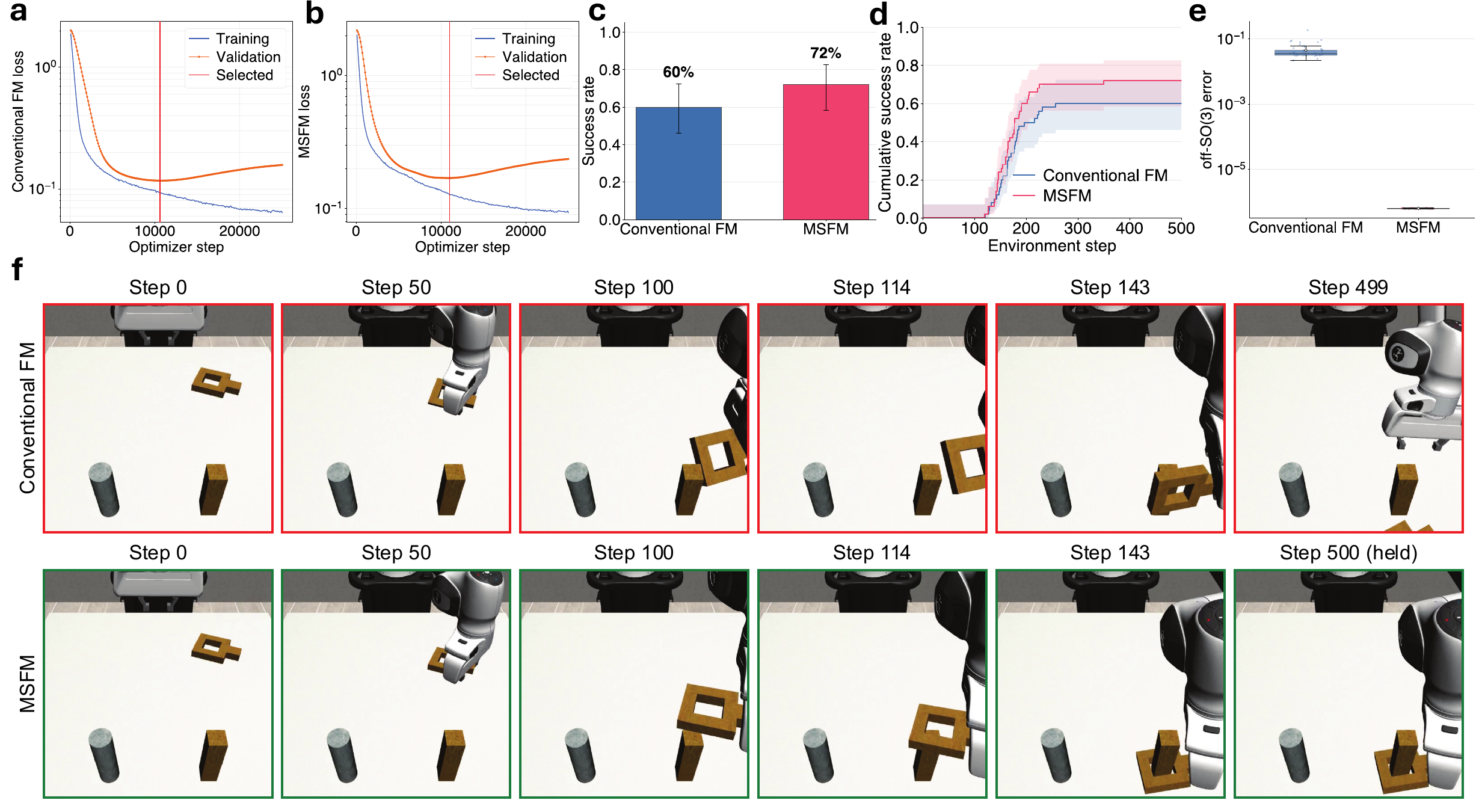}
    \caption{Robomimic Square Task. (a,b) Training/velocity-validation losses; red lines mark the checkpoint selection. (c) Success on paired nominal rollouts using best integrated-policy checkpoints.
    (d) Cumulative success rate vs. Robomimic environment time step.
    (e) Pre-projection rotation deviation, nearly five orders of magnitude lower with MSFM.
    (f) shows a completed episode from MSFM, while a conventional FM failed. 
    Here, the manifold for rotation poses is SO(3).}
    \label{fig:square}
\end{figure*}

\section{Conclusion}
\label{sec:conclusion}

We introduced manifold-stable flow matching (MSFM), which separates learned tangential transport from prescribed normal contraction. The formulation accommodates both known manifolds, through analytical projectors, and unknown geometry, through local affine proxies. We derived compatible probability paths and decomposed the regression loss into a learnable tangential term and a normal compatibility residual. Under the stated geometric and solution assumptions, we established manifold invariance and explicit transverse convergence bounds for known manifolds and fixed proxies, together with a proxy-union guarantee under minimum-residual selection. With a suitable contraction schedule, these bounds ensure vanishing normal deviation as flow time approaches its terminal value, without requiring exact learning of the tangential dynamics.

Our work complements the notion of manifold adherence discussed by \citet{pan2026muchado}: producing actions consistent with expert-supported geometry under unfamiliar observations can be more relevant to closed-loop performance than reducing validation error alone. MSFM makes attraction toward the chosen geometry an explicit property of the generative dynamics. In Push-T, this geometry is estimated from demonstrated action horizons; in Square, it includes the analytical rotation manifold $\SO(3)$. The stronger geometric adherence and higher observed task performance of MSFM are consistent with the importance of manifold adherence.

The absolute success rates in Push-T and Square are below those reported
for state-of-the-art generative policies (e.g., Chi et al., 2023). This gap
reflects our experimental design rather than a limitation of MSFM. Our goal
is a controlled comparison: FM and MSFM share identical temporal U-Net
backbones, conditioning, training budgets, and field-evaluation counts, so
that observed differences can be attributed to the manifold-stable
construction rather than to architecture or tuning. To this end, we use
deliberately compact networks (156K parameters for Push-T and 859K for
Square), substantially smaller than the backbones typically used in
diffusion-policy benchmarks. In addition, Push-T is evaluated under
repeated block disturbances, a protocol that differs from standard
benchmarks, so absolute numbers are not directly comparable. Because MSFM
modifies only the output field through projection and prescribed normal
contraction, it is agnostic to the backbone and can be combined with larger
architectures and longer training. Evaluating whether its geometric
advantages persist at state-of-the-art scale is left to future work.

{The experiments show geometric adherence in ellipse sampling, higher observed success and reward under Push-T disturbances, and improved Square success with substantially smaller rotation deviations. Together, these results support transverse stabilization as a useful complement to learned generative transport.}

\subsection*{AI use statement}
Generative AI assisted with language editing and final code polishing, optimizing readability and architectural consistency. 
We take responsibility for the final content of this work, including text, claims, and all numerical results.



\subsection*{Reproducibility statement}
Appendix~\ref{app:lie_groups} provides a special rotation-matrix formulation. 
Appendix~\ref{app:compatible_paths} discusses path derivations.
Appendix~\ref{app:proofs} gives proofs, and  Appendix~\ref{app:algorithms} provides more information on the algorithms. Appendices~\ref{app:ellipse}, \ref{app:pusht}, and~\ref{app:square} document data, optimization, checkpoint selection, and solvers for our numerical experiments.

\bibliography{ref.bib}

@inproceedings{lipman2023flowmatching,
title={Flow Matching for Generative Modeling},
author={Yaron Lipman and Ricky T. Q. Chen and Heli Ben-Hamu and Maximilian Nickel and Matthew Le},
booktitle={The Eleventh International Conference on Learning Representations },
year={2023},
url={https://openreview.net/forum?id=PqvMRDCJT9t}
}

@article{albergo2025stochastic,
  title   = {Stochastic Interpolants: A Unifying Framework for Flows and Diffusions},
  author  = {Albergo, Michael and Boffi, Nicholas M. and Vanden-Eijnden, Eric},
  journal = {Journal of Machine Learning Research},
  volume  = {26},
  number  = {209},
  pages   = {1--80},
  year    = {2025}
}

@article{ding2025rfmp,
  author={Ding, Haoran and Jaquier, Noémie and Peters, Jan and Rozo, Leonel},
  journal={IEEE Transactions on Robotics}, 
  title={Fast and Robust Visuomotor Riemannian Flow Matching Policy}, 
  year={2025},
  volume={41},
  number={},
  pages={5327-5343}
}

@inproceedings{braun2024rfmp,
  title={Riemannian flow matching policy for robot motion learning},
  author={Braun, Max and Jaquier, No{\'e}mie and Rozo, Leonel and Asfour, Tamim},
  booktitle={2024 IEEE/RSJ International Conference on Intelligent Robots and Systems (IROS)},
  pages={5144--5151},
  year={2024},
  organization={IEEE}
}

@inproceedings{sprague2024stablefm,
title={Incorporating Stability Into Flow Matching},
author={Christopher Iliffe Sprague and Arne Elofsson and Hossein Azizpour},
booktitle={ICML 2024 Workshop on Structured Probabilistic Inference {\&} Generative Modeling},
year={2024},
}

@inproceedings{beik2024ncds,
 author = {Beik Mohammadi, Hadi and Hauberg, S\o ren and Arvanitidis, Georgios and Figueroa, Nadia and Neumann, Gerhard and Rozo, Leonel},
 booktitle = {International Conference on Learning Representations},
 pages = {49097--49120},
 title = {Neural Contractive Dynamical Systems},
 volume = {2024},
 year = {2024}
}

@inproceedings{jaffe2024elcd,
 author = {Jaffe, Sean and Davydov, Alexander and Lapsekili, Deniz and Singh, Ambuj K. and Bullo, Francesco},
 booktitle = {Advances in Neural Information Processing Systems},
 pages = {66204--66225},
 publisher = {Curran Associates, Inc.},
 title = {Learning Neural Contracting Dynamics: Extended Linearization and Global Guarantees},
 volume = {37},
 year = {2024}
}

@article{lohmiller1998contraction,
  title   = {On Contraction Analysis for Non-linear Systems},
  author  = {Lohmiller, Winfried and Slotine, Jean-Jacques E.},
  journal = {Automatica},
  volume  = {34},
  number  = {6},
  pages   = {683--696},
  year    = {1998},
  doi     = {10.1016/S0005-1098(98)00019-3}
}

@article{tsukamoto2021tutorial,
  title   = {Contraction Theory for Nonlinear Stability Analysis and Learning-Based Control: A Tutorial Overview},
  author  = {Tsukamoto, Hiroyasu and Chung, Soon-Jo and Slotine, Jean-Jacques E.},
  journal = {Annual Reviews in Control},
  volume  = {52},
  pages   = {135--169},
  year    = {2021},
}

@inproceedings{pan2026muchado,
 author = {Pan, Chaoyi and Anantharaman, Giridharan and Huang, Nai-Chieh and Jin, Claire and Pfrommer, Daniel and Yuan, Chenyang and Permenter, Frank and Qu, Guannan and Boffi, Nicholas and Shi, Guanya and Simchowitz, Max},
 booktitle = {International Conference on Learning Representations},
 pages = {90575--90614},
 title = {Much Ado About Noising: Dispelling the Myths of Generative Robotic Control},
 volume = {2026},
 year = {2026}
}

@article{chi2023diffusionpolicy,
  title={Diffusion policy: Visuomotor policy learning via action diffusion},
  author={Chi, Cheng and Xu, Zhenjia and Feng, Siyuan and Cousineau, Eric and Du, Yilun and Burchfiel, Benjamin and Tedrake, Russ and Song, Shuran},
  journal={The International Journal of Robotics Research},
  volume={44},
  number={10-11},
  pages={1684--1704},
  year={2025},
  publisher={Sage Publications Sage UK: London, England}
}

@Book{FB-CTDS,
  author =    {F. Bullo},
  title =     {Contraction Theory for Dynamical Systems},
  year =      2026,
  edition =   {{1.3}},
  publisher = {Kindle Direct Publishing},
  ISBN =      {979-8836646806},
  url =       {https://fbullo.github.io/ctds},
}

@article{wang2005partial,
  title={On partial contraction analysis for coupled nonlinear oscillators},
  author={Wang, Wei and Slotine, Jean-Jacques E},
  journal={Biological cybernetics},
  volume={92},
  number={1},
  pages={38--53},
  year={2005},
  publisher={Springer}
}

@article{slotine2003modular,
  title={Modular stability tools for distributed computation and control},
  author={Slotine, Jean-Jacques E},
  journal={International Journal of Adaptive Control and Signal Processing},
  volume={17},
  number={6},
  pages={397--416},
  year={2003},
  publisher={Wiley Online Library}
}

@article{slotine2001modularity,
title = {Modularity, evolution, and the binding problem: a view from stability theory},
journal = {Neural Networks},
volume = {14},
number = {2},
pages = {137-145},
year = {2001},
issn = {0893-6080},
author = {J.-J.E. Slotine and W. Lohmiller},
}

@article{lohmiller2000nonlinear,
  title={Nonlinear process control using contraction theory},
  author={Lohmiller, Winfried and Slotine, Jean-Jacques E},
  journal={AIChE journal},
  volume={46},
  number={3},
  pages={588--596},
  year={2000},
  publisher={Wiley Online Library}
}

@misc{huggingface_gym_pusht,
  title        = {{gym-pusht}: A Gymnasium Environment for Push-T},
  author       = {{Hugging Face}},
  year         = {2024},
  howpublished = {\url{https://github.com/huggingface/gym-pusht}},
  note         = {Accessed 2026-07-08}
}

@article{nazerian2024bridging,
  title={Bridging the Gap between Reactivity, Contraction, and Finite-Time Lyapunov Exponents},
  author={Nazerian, Amirhossein and Sorrentino, Francesco and Aminzare, Zahra},
  journal={arXiv preprint arXiv:2410.23435},
  year={2024}
}

@inproceedings{mandlekar2022robomimic,
  title = 	 {What Matters in Learning from Offline Human Demonstrations for Robot Manipulation},
  author =       {Mandlekar, Ajay and Xu, Danfei and Wong, Josiah and Nasiriany, Soroush and Wang, Chen and Kulkarni, Rohun and Fei-Fei, Li and Savarese, Silvio and Zhu, Yuke and Mart\'in-Mart\'in, Roberto},
  booktitle = 	 {Proceedings of the 5th Conference on Robot Learning},
  pages = 	 {1678--1690},
  year = 	 {2022},
  volume = 	 {164},
  series = 	 {Proceedings of Machine Learning Research},
  month = 	 {08--11 Nov},
  publisher =    {PMLR},
}

@article{zhu2020robosuite,
  title={robosuite: A Modular Simulation Framework and Benchmark for Robot Learning},
  author={Zhu, Yuke and Wong, Josiah and Mandlekar, Ajay and Martín-Martín, Roberto and Joshi, Abhishek and Lin, Kevin and Maddukuri, Abhiram and Nasiriany, Soroush and Zhu, Yifeng},
  journal={arXiv preprint arXiv:2009.12293v3},
  year={2020}
}

@inproceedings{zaghen2026riemannian,
 author = {Zaghen, Olga and Eijkelboom, Floor and Pouplin, Alison and Liu, Cong and Welling, Max and van de Meent, Jan-Willem and Bekkers, Erik},
 booktitle = {International Conference on Learning Representations},
 pages = {34243--34286},
 title = {Riemannian Variational Flow Matching for Material and Protein Design},
 volume = {2026},
 year = {2026}
}

@inproceedings{zaghen2025towards,
title={Towards Variational Flow Matching on General Geometries},
author={Olga Zaghen and Floor Eijkelboom and Alison Pouplin and Erik J Bekkers},
booktitle={ICLR 2025 Workshop on Deep Generative Model in Machine Learning: Theory, Principle and Efficacy},
year={2025},
}

@inproceedings{
zhong2026riemannian,
title={Riemannian MeanFlow for One-Step Generation on Manifolds},
author={Zichen Zhong and Haoliang Sun and Yukun Zhao and Yongshun Gong and Yilong Yin},
booktitle={Forty-third International Conference on Machine Learning},
year={2026},
url={https://openreview.net/forum?id=DeOm4Axr9W}
}

@inproceedings{de2024pullback,
title={Pullback Flow Matching on Data Manifolds},
author={Friso de Kruiff and Erik J Bekkers and Ozan {\"O}ktem and Carola-Bibiane Sch{\"o}nlieb and Willem Diepeveen},
booktitle={ICML 2025 Generative AI and Biology (GenBio) Workshop},
year={2025}
}

@inproceedings{kapusniak2024metric,
 author = {Kapu\'{s}niak, Kacper and Potaptchik, Peter and Reu, Teodora and Zhang, Leo and Tong, Alexander and Bronstein, Michael and Bose, Avishek Joey and Di Giovanni, Francesco},
 booktitle = {Advances in Neural Information Processing Systems},
 pages = {135011--135042},
 publisher = {Curran Associates, Inc.},
 title = {Metric Flow Matching for Smooth Interpolations on the Data Manifold},
 volume = {37},
 year = {2024}
}

@inproceedings{atanackovic2025meta,
  title={Meta flow matching: Integrating vector fields on the wasserstein manifold},
  author={Atanackovic, Lazar and Zhang, Xi Nicole and Amos, Brandon and Blanchette, Mathieu and Lee, Leo J and Bengio, Yoshua and Tong, Alexander and Neklyudov, Kirill},
  booktitle={International Conference on Learning Representations},
  volume={2025},
  pages={94586--94610},
  year={2025}
}

@article{aminzare2013logarithmic,
  title={Logarithmic Lipschitz norms and diffusion-induced instability},
  author={Aminzare, Zahra and Sontag, Eduardo D},
  journal={Nonlinear Analysis: Theory, Methods \& Applications},
  volume={83},
  pages={31--49},
  year={2013},
  publisher={Elsevier}
}

@inproceedings{aminzare2014contraction,
  title={Contraction methods for nonlinear systems: A brief introduction and some open problems},
  author={Aminzare, Zahra and Sontag, Eduardo D},
  booktitle={53rd IEEE Conference on Decision and Control},
  pages={3835--3847},
  year={2014},
  organization={IEEE}
}

@inproceedings{bamberger2026carre,
  title={Carr{\'e} du champ flow matching: better quality-generalisation tradeoff in generative models},
  author={Bamberger, Jacob and Jones, Iolo and Duncan, Dennis and Bronstein, Michael and Vandergheynst, Pierre and Gosztolai, Adam},
  booktitle={International Conference on Learning Representations},
  volume={2026},
  pages={135754--135779},
  year={2026}
}

@inproceedings{jiang2026bures,
 author = {Jiang, Keyue and Cui, Jiahao and Dong, Xiaowen and Toni, Laura},
 booktitle = {International Conference on Learning Representations},
 pages = {140522--140560},
 title = {Bures-Wasserstein Flow Matching for Graph Generation},
 volume = {2026},
 year = {2026}
}

@inproceedings{li2026gauge,
title={Gauge Flow Matching: Efficient Constrained Generative Modeling over General Convex Set and Beyond},
author={Xinpeng Li and Enming Liang and Minghua Chen},
booktitle={The Fourteenth International Conference on Learning Representations},
year={2026},
url={https://openreview.net/forum?id=vxq1OnaAMq}
}

@inproceedings{ICLR2026_d33a2d61,
 author = {Guan, Yunrui and Balasubramanian, Krishna and Ma, Shiqian},
 booktitle = {International Conference on Learning Representations},
 pages = {130098--130124},
 title = {Mirror Flow Matching with Heavy-Tailed Priors for Generative Modeling on Convex Domains},
 volume = {2026},
 year = {2026}
}

@inproceedings{ICLR2026_58b1b19b,
 author = {Yang, Jeongyong and Jang, Seunghwan and Han, SooJean},
 booktitle = {International Conference on Learning Representations},
 pages = {54058--54087},
 title = {SafeFlowMatcher: Safe and Fast Planning using Flow Matching with Control Barrier Functions},
 volume = {2026},
 year = {2026}
}

@inproceedings{ICLR2026_238e6167,
 author = {Wyrwal, Kacper and Ceylan, Ismail I and Tong, Alexander},
 booktitle = {International Conference on Learning Representations},
 pages = {21075--21100},
 title = {Topological Flow Matching},
 volume = {2026},
 year = {2026}
}

@inproceedings{ICLR2026_a57483b3,
 author = {Baldan, Giacomo and Liu, Qiang and Guardone, Alberto and Thuerey, Nils},
 booktitle = {International Conference on Learning Representations},
 pages = {101215--101243},
 title = {Physics vs Distributions: Pareto Optimal Flow Matching with Physics Constraints},
 volume = {2026},
 year = {2026}
}

@inproceedings{ICLR2026_e8b782d0,
 author = {Tauberschmidt, Jan and Fellenz, Sophie and Vollmer, Sebastian and Duncan, Andrew},
 booktitle = {International Conference on Learning Representations},
 editor = {C. Vondrick and B. Hariharan and C. Raffel and L. Pinto and D. Yang and A. Faust},
 pages = {143992--144037},
 title = {Physics-Constrained Fine-Tuning of Flow-Matching Models for Generation and Inverse Problems},
 volume = {2026},
 year = {2026}
}

@misc{kumar2026flow,
      title={Flow Matching is Adaptive to Manifold Structures}, 
      author={Shivam Kumar and Yixin Wang and Lizhen Lin},
      year={2026},
      eprint={2602.22486},
      archivePrefix={arXiv},
      primaryClass={stat.ML},
      url={https://arxiv.org/abs/2602.22486}, 
}

@inproceedings{habashy2026geodesic,
title={Geodesic Flow Matching for Denoising High-Dimensional Structured Representations},
author={Karim Habashy and Chris Eliasmith},
booktitle={Forty-third International Conference on Machine Learning},
year={2026},
url={https://openreview.net/forum?id=a9CuW1f0CT}
}

@inproceedings{cai2026improvingclassifierfreeguidanceflow,
title={Improving Classifier-Free Guidance of Flow Matching via Manifold Projection},
author={Jian-Feng Cai and Haixia Liu and Zhengyi Su and Chao Wang},
booktitle={Forty-third International Conference on Machine Learning},
year={2026},
url={https://openreview.net/forum?id=IPp3LD6u16}
}

@inproceedings{wu2026flow,
title={Flow for Future: Geometric {SE}(3)-Equivariant Flow Matching for 3D Trajectory Prediction},
author={Junwei Wu and Yihang Liu and Ruixuan Yu and Jian Sun},
booktitle={Forty-third International Conference on Machine Learning},
year={2026},
}

@inproceedings{ryu2024diffusionedfs,
  author={Ryu, Hyunwoo and Kim, Jiwoo and An, Hyunseok and Chang, Junwoo and Seo, Joohwan and Kim, Taehan and Kim, Yubin and Hwang, Chaewon and Choi, Jongeun and Horowitz, Roberto},
  booktitle={2024 IEEE/CVF Conference on Computer Vision and Pattern Recognition (CVPR)}, 
  title={Diffusion-EDFs: Bi-Equivariant Denoising Generative Modeling on SE(3) for Visual Robotic Manipulation}, 
  year={2024},
  volume={},
  number={},
  pages={18007-18018},
}

@InProceedings{chatzipantazis2025stride,
  title = 	 {STRiDE: STate-space Riemannian Diffusion for Equivariant Planning},
  author =       {Chatzipantazis, Evangelos and Rao, Nishanth and Daniilidis, Kostas},
  booktitle = 	 {Proceedings of the 7th Annual Learning for Dynamics$\backslash$\& Control Conference},
  pages = 	 {1338--1352},
  year = 	 {2025},
  volume = 	 {283},
  series = 	 {Proceedings of Machine Learning Research},
  month = 	 {04--06 Jun},
  publisher =    {PMLR},
}

@inproceedings{tie2025etseed,
 author = {Tie, Chenrui and Chen, Yue and Wu, Ruihai and Dong, Boxuan and Li, Zeyi and Gao, Chongkai and Dong, Hao},
 booktitle = {International Conference on Learning Representations},
 pages = {60114--60132},
 title = {ET-SEED: EFFICIENT TRAJECTORY-LEVEL SE(3) EQUIVARIANT DIFFUSION POLICY},
 volume = {2025},
 year = {2025}
}

@inproceedings{wang2025equivariantpolicy,
  title = 	 {Equivariant Diffusion Policy},
  author =       {Wang, Dian and Hart, Stephen and Surovik, David and Kelestemur, Tarik and Huang, Haojie and Zhao, Haibo and Yeatman, Mark and Wang, Jiuguang and Walters, Robin and Platt, Robert},
  booktitle = 	 {Proceedings of The 8th Conference on Robot Learning},
  pages = 	 {48--69},
  year = 	 {2025},
  volume = 	 {270},
  series = 	 {Proceedings of Machine Learning Research},
  month = 	 {06--09 Nov},
  publisher =    {PMLR},
}

@inproceedings{yang2025equibot,
  title = 	 {EquiBot: SIM(3)-Equivariant Diffusion Policy for Generalizable and Data Efficient Learning},
  author =       {Yang, Jingyun and Cao, Ziang and Deng, Congyue and Antonova, Rika and Song, Shuran and Bohg, Jeannette},
  booktitle = 	 {Proceedings of The 8th Conference on Robot Learning},
  pages = 	 {1048--1068},
  year = 	 {2025},
  volume = 	 {270},
  series = 	 {Proceedings of Machine Learning Research},
  month = 	 {06--09 Nov},
  publisher =    {PMLR},
}

@inproceedings{zhu2025spherical,
  title = 	 {{SE}(3)-Equivariant Diffusion Policy in Spherical {F}ourier Space},
  author =       {Zhu, Xupeng and Wang, Fan and Walters, Robin and Shi, Jane},
  booktitle = 	 {Proceedings of the 42nd International Conference on Machine Learning},
  pages = 	 {80187--80206},
  year = 	 {2025},
  volume = 	 {267},
  series = 	 {Proceedings of Machine Learning Research},
  month = 	 {13--19 Jul},
  publisher =    {PMLR},
}

@inproceedings{bouvier2025ddat,
    AUTHOR    = {Jean-Baptiste Bouvier AND Kanghyun Ryu AND Qiayuan Liao AND Koushil Sreenath AND Negar Mehr}, 
    TITLE     = {{DDAT: Diffusion Policies Enforcing Dynamically Admissible Robot Trajectories}}, 
    BOOKTITLE = {Proceedings of Robotics: Science and Systems}, 
    YEAR      = {2025}, 
    ADDRESS   = {LosAngeles, CA, USA}, 
    MONTH     = {June}
}

@misc{foland2025kinematic,
      title={How Well do Diffusion Policies Learn Kinematic Constraint Manifolds?}, 
      author={Lexi Foland and Thomas Cohn and Adam Wei and Nicholas Pfaff and Boyuan Chen and Russ Tedrake},
      year={2025},
      eprint={2510.01404},
      archivePrefix={arXiv},
      primaryClass={cs.RO},
      url={https://arxiv.org/abs/2510.01404}, 
}

@InProceedings{chisari2025pointflowmatch,
  title = 	 {Learning Robotic Manipulation Policies from Point Clouds with Conditional Flow Matching},
  author =       {Chisari, Eugenio and Heppert, Nick and Argus, Max and Welschehold, Tim and Brox, Thomas and Valada, Abhinav},
  booktitle = 	 {Proceedings of The 8th Conference on Robot Learning},
  pages = 	 {982--993},
  year = 	 {2025},
  volume = 	 {270},
  series = 	 {Proceedings of Machine Learning Research},
  month = 	 {06--09 Nov},
  publisher =    {PMLR},
}

@inproceedings{zhang2026e3flow,
    author    = {Zhang, Qinglun and Cheng, Shen and Dan, Tian and Fan, Haoqiang and Liu, Guanghui and Liu, Shuaicheng},
    title     = {Efficient Hybrid SE(3)-Equivariant Visuomotor Flow Policy via Spherical Harmonics for Robot Manipulation},
    booktitle = {Proceedings of the IEEE/CVF Conference on Computer Vision and Pattern Recognition (CVPR)},
    month     = {June},
    year      = {2026},
    pages     = {27989-27998}
}

@inproceedings{yan2025maniflow,
  title = 	 {ManiFlow: A General Robot Manipulation Policy via Consistency Flow Training},
  author =       {Yan, Ge and Zhu, Jiyue and Deng, Yuquan and Yang, Shiqi and Qiu, Ri-Zhao and Cheng, Xuxin and Memmel, Marius and Krishna, Ranjay and Goyal, Ankit and Wang, Xiaolong and Fox, Dieter},
  booktitle = 	 {Proceedings of The 9th Conference on Robot Learning},
  pages = 	 {2268--2293},
  year = 	 {2025},
  volume = 	 {305},
  series = 	 {Proceedings of Machine Learning Research},
  month = 	 {27--30 Sep},
  publisher =    {PMLR},
}

@inproceedings{ouyang2026rfmpose,
 author = {Ouyang, Wenzhe and Ye, Qi and Wang, Jinghua and Xu, Zenglin and Chen, Jiming},
 booktitle = {Advances in Neural Information Processing Systems},
 pages = {67591--67610},
 publisher = {Curran Associates, Inc.},
 title = {RFMPose: Generative Category-level Object Pose Estimation via Riemannian Flow Matching},
 volume = {38, Main Conference},
 year = {2025}
}

@article{balcerak2026energy,
  title={Energy matching: Unifying flow matching and energy-based models for generative modeling},
  author={Balcerak, Michal and Amiranashvili, Tamaz and Terpin, Antonio and Shit, Suprosanna and Bogensperger, Lea and Kaltenbach, Sebastian and Koumoutsakos, Petros and Menze, Bjoern},
  journal={Advances in Neural Information Processing Systems},
  volume={38},
  pages={8583--8609},
  year={2026}
}

@inproceedings{chen2023riemannian,
 author = {Chen, Ricky T. Q. and Lipman, Yaron},
 booktitle = {International Conference on Learning Representations},
 pages = {47922--47945},
 title = {Flow Matching on General Geometries},
 volume = {2024},
 year = {2024}
}

@inproceedings{pi2026learning,
title={Learning Manifold Data with Flow Matching},
author={Sophia Pi and Mingcheng Lu and Jerry Yao-Chieh Hu and Maojiang Su and Weimin Wu and Han Liu},
booktitle={ICML 2026 Workshop on Foundations of Deep Generative Models: Understanding Memorization, Generalization, and Reasoning},
year={2026}
}
\bibliographystyle{iclr2027_conference}

\clearpage
\appendix

\section{Rotation-Matrix Specialization}
\label{app:lie_groups}

For an ambient matrix $X\in\R^{3\times3}$, let $R=\Pi_{\SO(3)}(X)$ be the nearest proper rotation. For a matrix direction $Z$, the analytical projectors are
\begin{align*}
    P_R(Z)&=R\skewop(R^\top Z),\\
    Q_R(Z)&=R\symop(R^\top Z).
\end{align*}
Here $\skewop(B)=(B-B^\top)/2$ and $\symop(B)=(B+B^\top)/2$, with orthogonality measured by the Frobenius inner product. 
The ambient MSFM specialization is
\begin{equation*}
    \dot X=P_R(W_\theta(t,X))-\alpha(t)(X-R),
    \qquad R=\Pi_{\SO(3)}(X).
\end{equation*}
Here $W_\theta$ denotes the network's $3\times3$ matrix output. Theorem~\ref{thm:known_main} applies wherever the nearest-rotation projection is smooth, and the solution remains in its domain. Square uses this field for each rotation block; the translation and gripper components are unconstrained.

\paragraph{Explicit conditional path.}
For an ambient source matrix $X_0$ and target rotation $R_1$, set $R_0=\Pi_{\SO(3)}(X_0)$, $S_0=R_0^\top(X_0-R_0)$, and $\Omega=\log(R_0^\top R_1)$, using a chosen rotation-logarithm branch. The matrix $S_0$ is symmetric and $\Omega$ is skew-symmetric. Along the geodesic $R_t=R_0e^{t\Omega}$, the solution of \eqref{eq:normal_path_ode_main} is
\begin{equation}
    r_t=q_\perp(t)R_0e^{t\Omega/2}S_0e^{t\Omega/2},
    \qquad X_t=R_t+r_t.
    \label{eq:so3_normal_path_appendix}
\end{equation}
Indeed, $R_t^\top r_t=q_\perp(t)e^{-t\Omega/2}S_0e^{t\Omega/2}$ is symmetric, and differentiation gives $Q_{R_t}\dot r_t=-\alpha(t)r_t$. Differentiating $Q_{R_t}r_t=r_t$ then recovers the full projector ODE. The training velocity is $U_t=\dot X_t$, evaluated analytically; no auxiliary numerical ODE solve is needed for this geometry. Compatibility with the sampling field requires $\Pi_{\SO(3)}(X_t)=R_t$, as in the tubular-neighborhood assumptions.

\section{Compatible Probability Paths in Tubular Coordinates}
\label{app:compatible_paths}

\begin{lemma}[Normality and compatibility of tubular paths]
\label{lem:compatible_paths}
Let $y:[0,1)\to\M$ be a continuously differentiable base curve whose orthogonal normal projector $Q_t=Q_{y_t}$ is continuously differentiable, and set $P_t=\Id-Q_t$. Let $\alpha:[0,1)\to(0,\infty)$ be locally integrable and $r_0\in N_{y_0}\M$. Then
\begin{equation*}
    \dot r_t=\dot Q_t r_t-\alpha(t)r_t,\qquad r(0)=r_0,
\end{equation*}
has a unique locally absolutely continuous solution. For every $t<1$,
\begin{equation*}
    Q_t r_t=r_t,\qquad
    \|r_t\|=q_\perp(t)\|r_0\|,\qquad
    q_\perp(t)=\exp\!\left(-\int_0^t\alpha(\tau)\,\dd\tau\right).
\end{equation*}
For $x_t=y_t+r_t$, its velocity $u_t=\dot x_t$ satisfies, almost everywhere,
\begin{align}
    P_tu_t&=\dot y_t+\dot Q_t r_t,
    \label{eq:tubular_tangent_velocity_appendix}\\
    Q_tu_t&=-\alpha(t)r_t. \notag
\end{align}
If, in addition, $x_t\in\mathcal U$ and $\Pi(x_t)=y_t$ for all $t<1$, then $\dist(x_t,\M)=q_\perp(t)\|r_0\|$, and the path has the same normal velocity as the sampling field \eqref{eq:known_vector_field}. Finally, if $y_t\to x_1\in\M$ and $\int_0^1\alpha(t)\,\dd t=\infty$, then $x_t\to x_1$ as $t\to1^-$.
\end{lemma}

\begin{proof}
On every compact interval $[0,T]\subset[0,1)$, the coefficient $\dot Q_t-\alpha(t)\Id$ is integrable. Standard existence and uniqueness for linear ODEs therefore give a unique absolutely continuous solution; uniqueness makes these solutions agree on overlapping intervals.

Differentiating $Q_t^2=Q_t$ and multiplying the resulting identity on the left and right by $Q_t$ gives
\begin{equation*}
    \dot Q_tQ_t+Q_t\dot Q_t=\dot Q_t,
    \qquad Q_t\dot Q_tQ_t=0.
\end{equation*}
Define the normality defect $z_t=(\Id-Q_t)r_t$. Almost everywhere,
\begin{equation*}
    \dot z_t=-Q_t\dot Q_t r_t-\alpha(t)z_t
    =-\bigl(Q_t\dot Q_t+\alpha(t)\Id\bigr)z_t,
    \qquad z_0=0,
\end{equation*}
where $r_t=Q_tr_t+z_t$ and $Q_t\dot Q_tQ_t=0$ were used. Uniqueness implies $z_t=0$ for every $t<1$, proving $Q_tr_t=r_t$. It follows that $Q_t\dot Q_t r_t=0$, so $\dot Q_t r_t$ is tangent at $y_t$ and is orthogonal to $r_t$. Hence
\begin{equation*}
    \frac{\dd}{\dd t}\|r_t\|^2=-2\alpha(t)\|r_t\|^2
\end{equation*}
almost everywhere. Integrating yields $\|r_t\|^2=q_\perp(t)^2\|r_0\|^2$ and thus the stated norm identity.

Differentiating $x_t=y_t+r_t$ gives $u_t=\dot y_t+\dot Q_t r_t-\alpha(t)r_t$. Since $\dot y_t$ and $\dot Q_t r_t$ are tangent while $r_t$ is normal, applying $P_t$ and $Q_t$ proves the velocity identities. Under the additional projection assumption, $r(x_t)=x_t-\Pi(x_t)=r_t$ and $Q_{\Pi(x_t)}=Q_t$. Consequently,
\begin{equation*}
    \dist(x_t,\M)=\|r_t\|,\qquad
    Q_{\Pi(x_t)}u_t=-\alpha(t)r(x_t)
    =Q_{\Pi(x_t)}v_\theta(t,x_t),
\end{equation*}
which establishes compatibility in the normal direction, without requiring the learned tangential velocity to equal the target.

Finally, the divergent integral implies $q_\perp(t)\to0$. Therefore
\begin{equation*}
    \|x_t-x_1\|\leq\|y_t-x_1\|+q_\perp(t)\|r_0\|\to0
    \qquad\text{as }t\to1^-.
\end{equation*}
\end{proof}

Normality alone does not imply $\Pi(x_t)=y_t$; the projection condition remains a separate geometric assumption. The term $\dot Q_t r_t$ accounts for the changing normal space and generally cannot be omitted. For the schedule \eqref{eq:singular_schedule_main}, $q_\perp(t)=e^{-\alpha_0t}(1-t)^\beta$.

For a fixed affine proxy, $\dot Q_t=0$ and $r_t=q_\perp(t)r_0$. With a general tangential schedule $q_{\parallel}$, the path is $x_t=x_1+q_{\parallel}(t)s_0+q_\perp(t)r_0$ and its velocity is $u_t=\dot q_{\parallel}(t)s_0-\alpha(t)q_\perp(t)r_0$. Choosing $q_{\parallel}(t)=1-t$ yields the path used in the main text. For \eqref{eq:singular_schedule_main}, the normal target magnitude is
\begin{equation*}
    \|Q_i u_t\|
    =e^{-\alpha_0t}\bigl[\alpha_0(1-t)+\beta\bigr]
    (1-t)^{\beta-1}\|r_0\|.
\end{equation*}
For nonzero $r_0$, this tends to zero as $t\to1^-$ when $\beta>1$, tends to $e^{-\alpha_0}\|r_0\|$ when $\beta=1$, and diverges when $0<\beta<1$. It is constant in time in the linear-path special case $\alpha_0=0$, $\beta=1$. On $0\leq t\leq1-\delta$, a uniform upper bound is $(\alpha_0+\beta)\max\{1,\delta^{\beta-1}\}\|r_0\|$.

\subsection{Minimum Kinetic Action of the Normal Motion}
\label{app:normal_kinetic_action}

Fix a smooth base path $y_t$ on $[0,1]$, with continuously differentiable normal projector $Q_t$, and an initial displacement $Q_0r_0=r_0$. For any absolutely continuous normal-displacement curve $\eta(t)$ satisfying $Q_t\eta(t)=\eta(t)$, define its normal velocity by $D_t^\perp\eta=Q_t\dot\eta$. The identity
\begin{equation*}
    \dot\eta=\dot Q_t\eta+D_t^\perp\eta
\end{equation*}
follows by differentiating $Q_t\eta=\eta$. The first term is tangent, while the second is normal. Thus $\|D_t^\perp\eta\|^2$ measures only normal motion, excluding the tangential velocity required by the changing normal space. This is the coordinate-free form of kinetic energy in a normal frame with no internal rotation.

\paragraph{Variational characterization.}
For a fixed smooth base path and $\alpha_0=0$, define the weighted normal kinetic action
\begin{equation}
    \mathcal E_\beta[\eta]
    =\frac12\int_0^1(1-t)^{1-\beta}
    \|Q_t\dot\eta(t)\|^2\,\dd t,
    \label{eq:normal_kinetic_action}
\end{equation}
on absolutely continuous normal-displacement curves $\eta(t)\in N_{y_t}\M$ satisfying $\eta(0)=r_0$ and $\eta(1)=0$. The projection excludes the tangential velocity required by the changing normal space. For $\beta=1$, the weight is one; only for fixed affine geometry does this also equal the ordinary ambient kinetic action of the displacement.

Take $\alpha_0=0$ and $\beta>0$. Let $r_t$ solve \eqref{eq:normal_path_ode_main}, and write $q_\perp(t)=(1-t)^\beta$. The vector $\xi_t=r_t/q_\perp(t)$, defined for $t<1$, satisfies $\dot\xi_t=\dot Q_t\xi_t$, $Q_t\xi_t=\xi_t$, and $\|\xi_t\|=\|r_0\|$; it extends continuously to $t=1$. Consequently,
\begin{equation*}
    D_t^\perp r_t=-\beta(1-t)^{\beta-1}\xi_t,
    \qquad
    \mathcal E_\beta[r]=\frac{\beta}{2}\|r_0\|^2.
\end{equation*}
Consider any other absolutely continuous normal curve $\eta$ with $\eta(0)=r_0$, $\eta(1)=0$, and finite action \eqref{eq:normal_kinetic_action}. Set $h_t=\eta(t)-r_t$, so $Q_th_t=h_t$ and $h_0=h_1=0$. With $w_\beta(t)=(1-t)^{1-\beta}$, expanding the action gives
\begin{equation*}
\begin{aligned}
    \mathcal E_\beta[\eta]-\mathcal E_\beta[r]
    ={}&-\beta\int_0^1\xi_t^\top D_t^\perp h_t\,\dd t +\frac12\int_0^1w_\beta(t)\|D_t^\perp h_t\|^2\,\dd t.
\end{aligned}
\end{equation*}
Because $\dot\xi_t$ is tangent and $h_t$ is normal,
$\xi_t^\top D_t^\perp h_t=\dd(\xi_t^\top h_t)/\dd t$.
The first integral therefore vanishes, leaving
\begin{equation*}
    \mathcal E_\beta[\eta]
    =\frac{\beta}{2}\|r_0\|^2
    +\frac12\int_0^1w_\beta(t)\|D_t^\perp h_t\|^2\,\dd t.
\end{equation*}
Equality requires $D_t^\perp h_t=0$ almost everywhere. Then $\dot h_t=\dot Q_t h_t$, and $h_0=0$ implies $h_t=0$ by uniqueness. Hence $r_t$ is the unique minimizer. For a fixed affine proxy, $Q_t$ is constant and $D_t^\perp\eta=\dot\eta$; the minimizer reduces to $(1-t)^\beta r_0$, recovering ordinary kinetic-action minimization when $\beta=1$.

\paragraph{Relaxation and ambient curvature.}
For the full schedule \eqref{eq:singular_schedule_main}, let $V(t)=\tfrac12\|r_t\|^2$. If $x_t=y_t+r_t$ remains in the tubular neighborhood with $\Pi(x_t)=y_t$, this is also $\tfrac12\dist(x_t,\M)^2$. The normal dynamics give
\begin{equation*}
    \dot V(t)=-2\alpha(t)V(t),
    \qquad
    V(t)=e^{-2\alpha_0t}(1-t)^{2\beta}V(0).
\end{equation*}
Under $\tau=-\log(1-t)$, $\dd V/\dd\tau=-2(\alpha_0e^{-\tau}+\beta)V$: the normal potential decays at a bounded rate that is constant when $\alpha_0=0$. Meanwhile, $\dot r_t=\dot Q_t r_t-\alpha(t)r_t$ includes the change in normal direction. Together with the moving base point, this can produce a curved ambient trajectory even when $\beta=1$. Tangential transport and normal contraction occur simultaneously, not as two successive phases.

The variational statement fixes $y_t$ and minimizes only the weighted normal action. It neither selects the tangential transport nor establishes optimality of the full generative path. Switching between affine proxies requires a separate analysis and does not inherit a global minimum-action guarantee from this result.

\paragraph{Including a constant contraction term.}
For the full schedule \eqref{eq:singular_schedule_main}, use $q_\perp(t)=e^{-\alpha_0t}(1-t)^\beta$ and the action
\begin{equation*}
    \mathcal E_q[\eta]
    =\frac12\int_0^1
    \frac{\|D_t^\perp\eta(t)\|^2}{\alpha(t)q_\perp(t)}\,\dd t.
\end{equation*}
The unique minimizer is again the corresponding solution $r_t$ of \eqref{eq:normal_path_ode_main}. Indeed, $\xi_t=r_t/q_\perp(t)$ still satisfies $\dot\xi_t=\dot Q_t\xi_t$, while $D_t^\perp r_t=-\alpha(t)q_\perp(t)\xi_t$. With $h=\eta-r$, the same expansion and endpoint argument yield
\begin{equation*}
    \mathcal E_q[\eta]
    =\frac12\|r_0\|^2
    +\frac12\int_0^1
    \frac{\|D_t^\perp h_t\|^2}{\alpha(t)q_\perp(t)}\,\dd t,
\end{equation*}
where $\int_0^1\alpha(t)q_\perp(t)\,\dd t=1$ was used. Uniqueness follows as above. When $\alpha_0=0$, this action equals $\mathcal E_\beta/\beta$ and has the same minimizer.

\section{Proofs}
\label{app:proofs}

\subsection{Proof of Proposition~\ref{prop:loss_decomposition}}

At $x_t$, take $P$, $Q$, and $r$ from either geometry in \eqref{eq:unifiednotation} in Proposition~\ref{prop:loss_decomposition}. In both cases, $P$ and $Q=\Id-P$ are complementary orthogonal projectors, with $Pr=0$ and $Qr=r$. Decompose $u_t=Pu_t+Qu_t$ to obtain
\begin{equation*}
\begin{aligned}
    v_\theta-u_t
    &=Pw_\theta-\alpha r-Pu_t-Qu_t\\
    &=P(w_\theta-u_t)-\bigl(Qu_t+\alpha r\bigr).
\end{aligned}
\end{equation*}
The first term lies in the tangent space and the second lies in the normal space. Their inner product is zero, so the Pythagorean identity gives \eqref{eq:loss_decomposition}. The normal term vanishes if and only if $Qu_t=-\alpha(t)r$ holds.

\subsection{Proof of Theorem~\ref{thm:known_main}}

If $x\in\M$, then $\Pi(x)=x$ and $r(x)=0$, so
\begin{equation*}
    v_\theta(t,x)=P_xw_\theta(t,x)\in T_x\M.
\end{equation*}
Thus $\M$ is invariant.

For the transverse bound, use $V_\M$ from \eqref{eq:distance_energy_main}. Along \eqref{eq:known_vector_field},
\begin{equation*}
\begin{aligned}
    \dot V_\M
    &=r^\top\left(P_yw_\theta-\alpha r\right) =-\alpha(t)\|r\|^2
    =-2\alpha(t)V_\M,
\end{aligned}
\end{equation*}
because $r\perp T_{\Pi(x)}\M$. Therefore,
\begin{equation*}
    V_\M(x(t))
    =
    \exp\!\left(-2\int_0^t\alpha(\tau)\,\dd\tau\right)
    V_\M(x(0)).
\end{equation*}
Taking square roots gives \eqref{eq:known_bound_main}. Divergence of the integral at $t=1$ implies terminal convergence.

\subsection{Proof of Theorem~\ref{thm:fixed_proxy_main}}

Let $r_i(t)=Q_i(x(t)-x_i^\ast)$. Since $Q_iP_i=0$ and $Q_i^2=Q_i$, we have $\dot r_i=Q_i\dot x=-\alpha(t)r_i$. Hence
\begin{equation*}
    r_i(t)
    =
    \exp\!\left(-\int_0^t\alpha(\tau)\,\dd\tau\right)r_i(0),
\end{equation*}
which proves \eqref{eq:fixed_proxy_bound_main}.

\subsection{Proof of Theorem~\ref{thm:global_proxy_main}}

For each proxy, define $V_i(x)=\tfrac12\|Q_i(x-x_i^\ast)\|^2$. On a compact time interval within $[0,1)$, the compositions $V_i(x(t))$ and their finite minimum are absolutely continuous. At almost every time, $\dot x$ exists and satisfies the selected field. For an active minimizing index $i$, $\nabla V_i^\top\dot x=-2\alpha(t)V_i$. The derivative of the minimum, whenever it exists, is bounded above by this active derivative. Hence $\dot V_{\mathrm{proxy}}\leq-2\alpha(t)V_{\mathrm{proxy}}$ almost everywhere, and the integral comparison inequality yields \eqref{eq:global_proxy_bound_main}. Moreover,
\begin{equation*}
    \dist(x,\proxyM)
    =\min_i\|Q_i(x-x_i^\ast)\|
    =\sqrt{2V_{\mathrm{proxy}}(x)},
\end{equation*}
which gives the distance bound.

\subsection{Proof of Corollary~\ref{cor:polynomial_bound}}

For \eqref{eq:singular_schedule_main},
$\int_0^t\alpha(\tau)\,\dd\tau=\alpha_0t-\beta\log(1-t)$.
Substituting $e^{-\int_0^t\alpha(\tau)\,\dd\tau}
=e^{-\alpha_0t}(1-t)^\beta$ into \eqref{eq:known_bound_main} and
\eqref{eq:fixed_proxy_bound_main} gives the stated bound for
$\mathcal S=\M$ and $\mathcal S=\widehat\M_i$, respectively; the latter
uses $\dist(x,\widehat\M_i)=\|Q_i(x-x_i^\ast)\|$.
Taking square roots in \eqref{eq:global_proxy_bound_main} gives the same
bound for $\mathcal S=\proxyM$. Setting $t=1-\delta$ gives the remaining case.

\section{Algorithm and Implementation Details}
\label{app:algorithms}

\paragraph{Local proxy preprocessing.}
Algorithm~\ref{alg:proxy_preprocessing} forms the neighbor-difference matrix $Y_i=[x_{i,1}^\ast-x_i^\ast,\ldots,x_{i,k}^\ast-x_i^\ast]$ for each anchor. Its singular value decomposition $Y_i=U_i\Sigma_iV_i^\top$ gives $T_i=U_i[:,1:d_i]$. The rank $d_i$ is fixed or selected by explained variance. The resulting $P_i=T_iT_i^\top$ and $Q_i=\Id-P_i$ are held fixed during network training.

\paragraph{Training and compatibility.}
Algorithm~\ref{alg:proxy_training} draws $x_0\sim p_0$, $x_1\sim p_{\mathrm{data}}$, and $t\sim\mathcal U[0,1-\delta]$. The constructed $(x_t,u_t)$ and the geometry used to evaluate the field are independent of the network parameters. Proposition~\ref{prop:loss_decomposition} therefore gives identical parameter gradients for full-field regression and its tangential term, but any nonzero normal residual must still be reported. For a target-associated affine path, exact compatibility must hold with the state-selected sampling proxy, not merely the target's proxy. Observation-conditioned policies use the same conditioning and proxy selector during training and sampling.

\paragraph{Sampling and time transformation.}
Algorithm~\ref{alg:proxy_inference} integrates the learned field. With $\tau=-\log(1-t)$, use $t=1-e^{-\tau}$ and
\begin{equation*}
    \frac{\dd x}{\dd\tau}=(1-t)v_\theta(t,x),
    \qquad (1-t)\alpha(t)=\alpha_0e^{-\tau}+\beta.
\end{equation*}
The normal-feedback coefficient lies between $\beta$ and $\alpha_0+\beta$ and is constant when $\alpha_0=0$. Stopping at $\tau=-\log\delta$ corresponds to $t=1-\delta$, with experiment-specific cutoffs given below. The scale satisfies $q_\perp(0)=1$ and $\dot q_\perp=-\alpha q_\perp$, so changing the rate requires updating both path states and velocities.

\paragraph{Linear-path compatibility.}
For a fixed affine proxy containing $x_1$, the compatible base path is $y_t=x_1+(1-t)s_0$. Setting $\alpha_0=0$ and $\beta=1$ recovers the ambient linear conditional path. Keeping that linear path for other schedule parameters instead produces the normal residual $Qu_t+\alpha(t)r_t=[\alpha_0(1-t)+\beta-1]r_0$. This fixed-proxy calculation does not extend automatically to changing proxies or curved manifolds.

\section{Additional Ellipse Details}
\label{app:ellipse}

For a generated point $z=(z_1,z_2)$, the analytic ellipse residual is $|(z_1/2)^2+z_2^2-1|$. The off-proxy metric is $\|Q_{i(z)}(z-x_{i(z)}^\ast)\|$, while nearest-data distance measures closeness to the sampled arc rather than to its extended affine proxies.

The analytic ellipse is used only to evaluate generated samples. The training data contain 800 uniformly spaced angles on $[-3\pi/4,\pi/2)$, mapped to $(2\cos\varphi,\sin\varphi)$. Each affine proxy has dimension one and uses 31 nearest neighbors; the active proxy is selected by nearest anchor, not minimum normal residual. The prior distribution is uniform on $[-3,3]\times[-2,2]$.

The network has two hidden layers of width 64 with tanh activations (4,546 parameters). Adam runs for 1,600 updates with learning rate $0.003$, and batch size 256. Training uses linear conditional paths with $\beta=1$. Validation uses 256 fixed source/target/time tuples drawn from the same 800 target points; it is not a held-out target-data split. The plotted full-field losses include the normal compatibility residual. Their nonzero plateau therefore cannot be interpreted as tangential regression error alone.

Sampling uses RK4 in log time to $t=1-10^{-6}$, with 700 independent source draws. Mean terminal off-proxy error is $1.10\times10^{-6}$, median $1.03\times10^{-6}$, and maximum $2.57\times10^{-6}$. Mean nearest-data distance is $2.75\times10^{-3}$; the mean absolute analytic ellipse residual is $9.74\times10^{-5}$. The latter two metrics distinguish attraction to the union of extended affine lines from closeness to the sampled arc.

\section{Push-T Experimental Details}
\label{app:pusht}

\paragraph{Data and representation.}
We use the state-based \texttt{pusht\_cchi\_v7\_replay} demonstrations. The 206 demonstrations are split into 164 training, 21 validation, and 21 held-out test episodes. The condition is $c=(p_{a,x},p_{a,y},p_{b,x},p_{b,y},\cos\theta_b,\sin\theta_b)$, where $p_a$ and $p_b$ are world-frame pusher and block positions and $\theta_b$ is block orientation. The targets are eight successive absolute pusher positions, not object-frame poses or angular actions. Horizons never cross episode boundaries; 1,442 incomplete horizons are discarded. The episode-disjoint training/validation/test split contains 19,232/2,469/2,507 complete horizons. Normalization and the proxy atlas use training data only; held-out test horizons are not used for checkpoint selection.

\paragraph{Proxy field and training objective.}
For each normalized training horizon $a_i$ and condition $c_i$, we select 64 neighbors using the joint squared distance $\|c-c_i\|^2+0.1\|x-a_i\|^2$. PCA of horizon differences defines the tangent projector $P_i$ by retaining at least $95\%$ of their squared singular-value energy; $Q_i=I-P_i$ is its normal complement. The median retained dimension is three. The active index $i=i(x,c)$ minimizes the same joint distance, using the current generated horizon and observed condition, without access to the target action. With flow time $t$, MSFM uses $v_\theta=P_iw_\theta(t,x,c)-(1-t)^{-1}Q_i(x-a_i)$. This run regresses against the linear-path velocity $u=x_1-x_0$ at $x_t=(1-t)x_0+tx_1$, with demonstrated horizon $x_1$ and $x_0\sim\mathcal N(0,I_{16})$. Because the selected proxy need not contain $x_1$, this is not an exactly normal-compatible probability path: its total loss includes a model-independent transverse residual. The plotted parallel loss is $\|P_i(v_\theta-u)\|^2/16$, recovered by subtracting that residual using the original seeded samples. Fixed validation samples make this subtraction constant across epochs, preserving checkpoint selection. The MSFM selected-checkpoint validation loss is approximately $0.173$ after subtraction, versus $31.520$ before subtraction. 

\paragraph{Optimization and sampling.}
Each temporal U-Net has 156,338 trainable parameters and channel widths $(16,32,64)$, with a learning rate of $5\times10^{-3}$. Both networks use ELU activations, kernel size five, eight normalization groups, one residual block per level, one middle block, and a 32-dimensional time embedding. 
Push-T uses episode-disjoint data splits and AdamW with batch size 256 and weight decay $10^{-5}$.
Training uses gradient-norm clipping at 10 and early-stopping patience of 500 epochs within a 3,000-epoch maximum. FM and MSFM stop at epochs 537 and 543, selecting epochs 37 and 43, respectively. Validation evaluates velocity regression, not integrated-action prediction. Each rollout plan starts from standard Gaussian noise and uses 20 RK4 steps to $t=1-10^{-6}$; FM integrates in $t$, whereas MSFM uses $s=-\log(1-t)$. Thus, field-evaluation counts match, but integration grids differ. Generated targets are denormalized, clipped to $[0,512]^2$, and only the first target is executed before replanning.

\paragraph{Evaluation and interpretation.}
The reported task evaluation is restricted to disturbed rollouts to probe performance under repeated external perturbations. Both methods use initial-state seeds 0--49 with a 1,000-step limit. Success means maximum environment reward of at least $0.95$; the reported score is the episode maximum reward averaged over rollouts. Episodes schedule up to ten block translations before step 900, drawing displacements from $\mathcal N(0,100I_2)$ in raw coordinates (standard deviation 10 per axis), clipping block position to the workspace and resetting linear/angular velocity. Success counts are 37/50 (FM) and 41/50 (MSFM), with mean maximum rewards of $0.862$ and $0.922$, respectively. Success bars show Wilson 95\% intervals; reward bars show means. These results demonstrate higher observed success and reward for MSFM under the tested disturbance protocol.

\section{Robomimic Square Experimental Details}
\label{app:square}

\paragraph{Data, conditioning, and controller.}
The experiment uses \texttt{square\_ph\_low\_dim.hdf5} and \texttt{NutAssemblySquare} in Robosuite 1.4.1, with a Panda robot, 20\,Hz control, and an absolute \texttt{OSC\_POSE} controller. The proficient-human dataset contains 200 demonstrations, split into 180 training and 20 validation episodes, yielding 27,213 and 2,941 horizons; there is no separate offline test partition. Two successive 59-dimensional observations give a 118-dimensional condition containing object state, end-effector position/quaternion and velocities, gripper positions/velocities, and joint positions, sine/cosine encodings, and velocities. The dataset and simulator use the same observation ordering and training-derived normalization; images are not used. Demonstration targets reconstruct the intended absolute controller goals from scaled incremental commands, not from subsequently achieved poses. For current end-effector pose $(p_a,R_a)$ and future world-frame goal $(p_h,R_h)$, each token is $(R_a^\top(p_h-p_a)/\sigma_p,\operatorname{vec}(R_a^\top R_h),g_h)$, where $g_h$ is the gripper command and $\sigma_p=0.106773$\,m is the training-derived translation scale. The anchor remains fixed within each plan. Horizons are padded at episode ends, with padded tokens excluded from losses.

\paragraph{Geometry and probability paths.}
Both methods learn in the same ambient space $\R^{64\times13}$, with target geometry $(\R^3\times\SO(3)\times\R)^{64}$. FM uses linear interpolation from standard Gaussian noise. MSFM uses the tubular path $M_t=R_t+r_t$, where $R_t$ is the geodesic from the projected source matrix to the demonstrated rotation and $r_t$ is the explicit solution in \eqref{eq:so3_normal_path_appendix}, with $q_\perp(t)=e^{-\alpha_0t}(1-t)^\beta$. The saved run uses $\alpha_0=0$ and $\beta=1$; translation and gripper paths remain linear. For a current matrix $M$, let $R=\Pi_{\SO(3)}(M)$ be its nearest proper rotation. The analytical tangent projector acts on a matrix direction $Z$ as $P_R(Z)=R\,\skewop(R^\top Z)$, and the rotational field is $P_R(w_\theta)-\alpha(t)(M-R)$ with $\alpha(t)=1/(1-t)$. Training regresses the learned tangent component and reports the analytical normal residual separately. Losses mask padded tokens and use equal position/rotation/gripper weights, with rotational matrix errors weighted by the canonical factor $1/2$. The tubular compatibility statement applies within the smooth projection neighborhood; ambient Gaussian initialization does not place intermediate states exactly on $\SO(3)$.

\paragraph{Optimization and checkpoint selection.}
Each temporal U-Net has 859,347 trainable parameters and channel widths $(32,60,124)$. Both networks use SiLU activations, kernel size three, four normalization groups, two residual blocks per level, two middle blocks, a 78-dimensional time embedding, and a 74-dimensional conditioning embedding. 
Square task uses episode-disjoint data splits and batch size 256.
Training uses identical initial parameters, AdamW, a 1,000-step learning-rate warmup followed by cosine decay from $3\times10^{-4}$ to $3\times10^{-5}$, gradient clipping at 10, and exponential-moving-average weights with decay $0.999$. The budget is 25,000 updates (234 epochs, with the last epoch partial). Rollout checkpoints minimize the integrated validation score over the first four actions, using 256 fixed validation examples and four Gaussian draws per example. This score averages squared normalized position error, squared rotation geodesic error after projection, and squared clipped-gripper error, divided by seven. The velocity-loss minima marked in Figure~\ref{fig:square}(a,b) occur at steps 10,593 and 11,021.

\paragraph{Sampling and paired evaluation.}
Both policies use RK4. FM integrates uniformly to $t=1$; MSFM integrates to $1-10^{-10}$. Raw generated matrices are composed with the fixed observation anchor; the decoder applies the same hand-to-legacy-controller rotation correction and one final $\SO(3)$ projection before converting to absolute axis-angle commands. The first four of 64 generated actions are executed before replanning. Evaluation uses the same saved bank of 50 initial simulator states, Gaussian source, and a 500-step limit, with no injected disturbances. Success is the simulator's nut-assembly success flag.

\paragraph{Results and scope.}
FM succeeds on 30/50 episodes and MSFM on 36/50; 25 pairs succeed with both, five only with FM, and eleven only with MSFM.
Geometric deviation is $\|M-\Pi_{\SO(3)}(M)\|_F$, averaged over the generated horizon and replanning times within each episode, then over episodes, before execution projection. Its mean is $0.0451$ for FM and $6.40\times10^{-7}$ for MSFM, with lower MSFM deviation in all 50 pairs. The residual measures rotational validity, not accuracy relative to the expert action.

\end{document}